\documentclass{article} 
\usepackage{iclr2027_conference,times}

\usepackage{amsmath,amsfonts,bm}

\def\eqref#1{equation~\ref{#1}}

\def\1{\bm{1}}

\DeclareMathAlphabet{\mathsfit}{\encodingdefault}{\sfdefault}{m}{sl}
\SetMathAlphabet{\mathsfit}{bold}{\encodingdefault}{\sfdefault}{bx}{n}

\newcommand{\R}{\mathbb{R}}

\DeclareMathOperator*{\argmin}{arg\,min}

\usepackage{physics}
\usepackage[utf8]{inputenc} 
\usepackage[T1]{fontenc}    
\usepackage{hyperref}       
\usepackage{url}            
\usepackage{booktabs}       
\usepackage{wrapfig}
\usepackage{amsfonts}       
\usepackage{nicefrac}       
\usepackage{microtype}      
\usepackage{xcolor}         

\usepackage{microtype}
\usepackage{graphicx}

\usepackage{booktabs} 
\usepackage{tabularx}
\usepackage{amsthm, amsmath, amssymb}
\usepackage{listings}
\usepackage{relsize}
\usepackage{mathtools}

\usepackage{dsfont}
\usepackage{hhline}

\usepackage{csquotes}
\usepackage{siunitx}
\usepackage{booktabs}
\numberwithin{equation}{section}
\usepackage{multirow}
\usepackage{arydshln} 

\usepackage{dsfont}
\usepackage{tikz}
\usetikzlibrary{arrows.meta, positioning, calc, fit} 
\usetikzlibrary{positioning}
\usepackage{amsfonts}
\usepackage{xcolor}
\usepackage{algorithm}
\usepackage{algorithmic}

\definecolor{statecolor}{RGB}{68, 119, 170} 
\definecolor{costatecolor}{RGB}{34, 136, 51} 

\usepackage{subcaption}
\usepackage{textgreek}
\usepackage[most]{tcolorbox}

\usepackage{hyperref}
\usepackage{url}
\theoremstyle{plain}
\newtheorem{theorem}{Theorem}[section]
\newtheorem{proposition}[theorem]{Proposition}

\theoremstyle{definition}
\newtheorem{definition}[theorem]{Definition}

\theoremstyle{remark}

\usepackage{booktabs}
\usepackage{pifont}
\usepackage{amssymb}
\usepackage{threeparttable}

\newcommand{\cmark}{\ding{51}}
\newcommand{\xmark}{\ding{55}}
\newcommand{\nomark}{\textendash}

\title{Hidden States as Value Gradients: The Pontryagin Structure of Recurrent Policies}

\author{David Leeftink, Max Hinne \& Marcel van Gerven \\
Department of Machine Learning and Neural Computing \\Donders Institute for Brain, Cognition and Behaviour\\
Radboud University, Nijmegen, the Netherlands. \\
\texttt{\{david.leeftink,max.hinne,marcel.vangerven\}@ru.nl}
}

\iclrfinalcopy 
\begin{document}

\maketitle
\lhead{Preprint}


\begin{abstract}

A key capability of intelligent agents is to act effectively under incomplete state observations. Recurrent policies address this by compressing observation histories into a hidden state. In this work, we show that the hidden state of a recurrent policy admits a control-theoretic interpretation: it plays the role of the co-state in Pontryagin's minimum principle, and the readout that maps it to actions implements control-Hamiltonian minimization. The hidden state thus tracks the gradient of the value function, encoding the optimality structure of the underlying control problem. We formalize this correspondence through a class of policies we refer to as co-state policies (CPs) and show that several modern recurrent cells implicitly realize this structure. The correspondence also allows for a co-state loss for actor-critic training, in which the critic's gradient serves as a target for the actor's hidden state. Empirically, we find that hidden states trained with the co-state loss encode co-state information beyond what is linearly decodable from the environment state alone, and that the loss improves policies both within and outside this class on challenging locomotion tasks, including the H1 and Berkeley humanoids. By connecting the minimum principle to recurrent memory, we provide a control-theoretic account of what hidden states compute in continuous control and a mechanism for shaping them toward optimality.
\end{abstract}

\section{Introduction}
A fundamental challenge for intelligent agents is to act effectively under partial observability. Physical environments are characterized by noisy sensors, measurement delays, or missing data, such that a single instantaneous observation rarely determines the underlying state of the system. To succeed, an agent must learn to integrate a history of observations and actions to infer what cannot be directly seen~\citep{kaelbling1998planning}. 
In deep reinforcement learning (RL), this is commonly handled by recurrent policies, which carry a hidden state that accumulates information over time. 

A natural question is what this memory should compute. The standard reading is that the hidden state tracks a belief over the unobserved state. This view is rooted in the sufficiency of the belief state~\citep{aastrom1965optimal}: a policy acting on the posterior over states loses nothing relative to one acting on the full history, so estimation and control may be treated as separate stages.

In this work, we develop a complementary account of what the hidden state represents, by drawing a formal link between recurrent policies and the Pontryagin minimum principle (PMP)~\citep{pontryagin1987}. We argue that the hidden state does not need to represent the full belief, but only the quantity that the optimal action depends on: the gradient of the value function, $\nabla_x V(x)$, which appears in Pontryagin's principle as the co-state. In this view, the hidden state is a statistic of the belief: the part on which the optimal action depends.

We formalize this relationship by introducing co-state policies (CPs), the class of recurrent policies whose memory mirrors the optimal co-state dynamics of the minimum principle. In PMP, the co-state evolves alongside the state and measures the sensitivity of the optimal value function to the current state, with dynamics that are affine in the co-state and driven by the running cost. In a CP, the hidden state plays the role of the co-state. The defining features of this class of architectures are (1) that the recurrence is affine in the hidden state, and (2) that the readout performs Hamiltonian minimization. 

We demonstrate that the CP class includes a wide range of recurrent cells, such as structured state space models (SSMs) and linear RNNs~\citep{gu2021efficiently, orvieto2023resurrecting}, while gated cells such as the LSTM and GRU fall outside the class. Furthermore, because the co-state and the value gradient coincide along the optimal trajectories~\citep{vinter1986costate}, the critic's gradient supplies a target for the actor's hidden state that can be used as an auxiliary objective in the actor-critic setting.
On a non-linear stabilization task with numerically computed optimal co-states, we find that under reward maximization alone most hidden states do not encode co-state structure beyond what is linearly decodable from the state, but that the co-state objective lifts nearly all of them above this ceiling. On humanoid locomotion (Unitree H1, Berkeley), the co-state objective improves policies both inside and outside the class.

Ultimately, this work asks what recurrent policies compute under uncertainty. By framing recurrent policies' hidden states as dynamic processes governed by the minimum principle, we provide a control-theoretic account of recurrent memory in deep RL, and a means of supervising it directly.


\section{Background: Recurrent Policies in Reinforcement Learning}\label{sec:2}

\begin{figure}[t]
\centering
\definecolor{cpfA}{RGB}{34,136,51}    
\definecolor{cpfB}{RGB}{170,51,119}   
\definecolor{cpfC}{RGB}{238,102,119}  
\definecolor{cpfD}{RGB}{112, 143, 189}  
\def\cpfa{0.33}\def\cpfw{1.55}\def\cpftmax{9}
\def\cpfX#1#2#3{#1*exp(-\cpfa*(#3))*cos(deg(\cpfw*(#3))+#2)}
\def\cpfY#1#2#3{#1*exp(-\cpfa*(#3))*sin(deg(\cpfw*(#3))+#2)}

\def\cpfL#1#2#3{0.88*\cpfX{#1}{#2}{#3} + 0.132*\cpfY{#1}{#2}{#3}}
\def\cpfM#1#2#3{0.132*\cpfX{#1}{#2}{#3} + 0.66*\cpfY{#1}{#2}{#3}}

\def\cpfspirals#1#2#3{%
  \foreach \cpfc/\cpfr/\cpfp in {cpfA/1.25/25, cpfB/1.05/150, cpfC/1.2/255}{%
    \ifnum#3=1 \def\cpfcol{\cpfc}\else\def\cpfcol{black!30}\fi
    \draw[\cpfcol, line width=#1, line cap=round]
      plot[domain=0:\cpftmax, samples=200, smooth, variable=\cpft]
      ({\cpfX{\cpfr}{\cpfp}{\cpft}}, {\cpfY{\cpfr}{\cpfp}{\cpft}});
    \ifnum#3=1
      \fill[\cpfc] ({\cpfr*cos(\cpfp)}, {\cpfr*sin(\cpfp)}) circle (1.6pt);
      \foreach \cpft/\cpfl in {0.55/5pt, 2.3/4pt}{%
        \draw[\cpfc, line width=#1, -{Stealth[length=\cpfl,width=0.8*\cpfl]}]
          ({\cpfX{\cpfr}{\cpfp}{\cpft}}, {\cpfY{\cpfr}{\cpfp}{\cpft}}) --
          ({\cpfX{\cpfr}{\cpfp}{\cpft+0.06}}, {\cpfY{\cpfr}{\cpfp}{\cpft+0.06}});
      }%
    \fi
  }%
}
\resizebox{\textwidth}{!}{%
\begin{tikzpicture}[x=1cm, y=1cm,
  ax/.style={-{Stealth[length=4pt]}, black!55, line width=0.5pt},
  lbl/.style={font=\footnotesize, black!85},
  pcap/.style={font=\footnotesize, black},
  peq/.style={font=\small, black!90},
  panel/.style={rounded corners=6pt, line width=0.7pt},
  link/.style={-{Stealth[length=6pt,width=5pt]}, line width=0.9pt, black!75}]

\draw[panel, draw=statecolor!70, fill=statecolor!6]     (-4.05,-2.05) rectangle (5.9,1.95);
\draw[panel, draw=costatecolor!70, fill=costatecolor!6] (7.8,-2.05) rectangle (15.6,1.95);
\node[font=\small\bfseries, statecolor!80!black]   at (0.93,1.65) {Pontryagin Minimum Principle};
\node[font=\small\bfseries, costatecolor!80!black] at (11.7,1.65) {Co-state Policies};

\node[pcap] at (-2.55,1.17) {State $x^\star(t)$};
\begin{scope}[shift={(-2.55,-0.33)}]
  \draw[ax] (-0.95,0) -- (1.0,0) node[right, lbl] {$x_1$};
  \draw[ax] (0,-0.9) -- (0,0.95)  node[right, lbl] {$x_2$};
  \begin{scope}[scale=0.62] \cpfspirals{0.9pt}{0}{1} \end{scope}
  \node[peq]  at (0.25,-1.30) {\footnotesize $\dot x^\star = f(x^\star) + g(x^\star)\,u^\star$};
\end{scope}
\draw[link] (-1.05,-0.05) -- node[above, font=\small] {$\nabla_x V^\star$} (0.0,-0.05);
\node[font=\scriptsize, black!80, align=center] at (-0.52,-0.5) {value\\gradient};

\node[pcap] at (2.8,1.17) {Co-state $\lambda^\star(t)$};
\begin{scope}[shift={(0.3,-0.33)}]
  \draw[black!18, line width=0.4pt, dashed, dash pattern=on 1.2pt off 1.2pt] (0,0.08) -- (2.4,0.08);
  \begin{scope}[shift={(0,0.08)}, xscale=2.3/7.0, yscale=0.78]
    \draw[cpfD, line width=0.9pt]
      plot[domain=0:7.0, samples=180, smooth, variable=\cpft] (\cpft, {\cpfL{1.25}{25}{\cpft}});
    \draw[cpfD, line width=0.9pt]
      plot[domain=0:7.0, samples=180, smooth, variable=\cpft] (\cpft, {\cpfM{1.25}{25}{\cpft}});
  \end{scope}
  \draw[ax] (0,-0.8) -- (2.55,-0.8) node[right, lbl] {$t$};
  \draw[ax] (0,-0.8) -- (0,1.12) node[right, lbl] {$\lambda_i^\star$};
  \node[peq]  at (2.6,-1.30) {$\dot\lambda^\star = -\textcolor{blue}{\nabla_x q} - \textcolor{teal}{(\nabla_x \dot x^\star)}\,\lambda^\star$};
\end{scope}

\begin{scope}[shift={(4.15,-0.33)}]
  \draw[ax] (-1.0,0) -- (1.05,0) node[right, lbl] {$\lambda_1^\star$};
  \draw[ax] (0,-0.95) -- (0,1.0)  node[right, lbl] {$\lambda_2^\star$};
  \begin{scope}[cm={0.72,0.108,0.108,0.54,(0,0)}] \cpfspirals{1.1pt}{0}{1} \end{scope}
\end{scope}

\begin{scope}[xshift=-1.95cm]
\draw[link] (8.0,-0.05) -- node[above, font=\small] {$h = T\lambda^\star$} (9.6,-0.05);
\node[font=\scriptsize, black!80, align=center] at (8.8,-0.63) {realizable\\(Prop.~3.3)};

\node[pcap] at (13.7,1.17) {Hidden state $h(t)$};
\begin{scope}[shift={(10.35,-0.25)}]
  \draw[black!18, line width=0.4pt, dashed, dash pattern=on 1.2pt off 1.2pt] (0,0.08) -- (2.95,0.08);
  \begin{scope}[shift={(0.0,0.08)}, xscale=2.9/7.0, yscale=0.85]
    \foreach \cpfamp/\cpfth/\cpfop in {1.00/10/0.85, 0.80/75/0.75, 0.62/140/0.68,
                                       0.46/205/0.6, 0.32/265/0.55, 0.20/325/0.5}{%
      \draw[costatecolor!78!black, line width=0.7pt, opacity=\cpfop]
        plot[domain=0:7.0, samples=180, smooth, variable=\cpft]
        (\cpft, {\cpfamp*(cos(\cpfth)*\cpfX{1.5}{25}{\cpft}
                 + sin(\cpfth)*\cpfY{1.25}{25}{\cpft})/1.25});
    }
  \end{scope}
  \draw[ax] (0,-0.8) -- (3.15,-0.8) node[right, lbl] {$t$};
  \draw[ax] (0,-0.8) -- (0,1.12) node[right, lbl] {$h_i$};
  \node[peq]  at (3.55,-1.38) {$\dot h = -\textcolor{blue}{b_\theta(y)} - \textcolor{teal}{F_\theta(y,u)}\,h$};
\end{scope}

\begin{scope}[shift={(15.45,-0.25)}]
  \draw[ax] (0,0) -- (1.55,-0.6)  node[right, lbl] {$h_1$};
  \draw[ax] (0,0) -- (0,1.15)     node[right, lbl] {$h_2$};
  \draw[ax] (0,0) -- (-0.95,-0.8) node[left, lbl]  {$h_3$};
  \begin{scope}[cm={0.62,-0.21,0.38,0.34,(0,0)}]
    \fill[costatecolor!45, fill opacity=0.25] (-1.45,-1.45) rectangle (1.45,1.45);
    \foreach \cpfg in {-0.725,0,0.725}{%
      \draw[costatecolor!50, line width=0.3pt] (\cpfg,-1.45) -- (\cpfg,1.45);
      \draw[costatecolor!50, line width=0.3pt] (-1.45,\cpfg) -- (1.45,\cpfg);}
    \draw[costatecolor!70!black, line width=0.5pt] (-1.45,-1.45) rectangle (1.45,1.45);
  \end{scope}
  \begin{scope}[cm={0.745,-0.175,0.416,0.246,(0,0)}] \cpfspirals{1.0pt}{0}{1} \end{scope}
  \node[font=\scriptsize, costatecolor!60!black] at (-1.15,0.6) {$\operatorname{im} T$};
\end{scope}
\end{scope}
\end{tikzpicture}}
\caption{\textbf{The co-state structure of recurrent policies.} (left) Along optimal trajectories, the co-state is the value gradient $\lambda^\star = \nabla_x V^\star$ and evolves affinely in $\lambda^\star$. (right) Co-state policies update their memory by a recurrence of the same affine form, so their $d_h$-dimensional hidden state can carry the $d_x$-dimensional co-state as a linear image $h = T\lambda^\star$ in $\mathbb{R}^{d_h}$. Colors mark corresponding terms.}
\label{fig:latent_costates}
\end{figure}

\subsection{Continuous control under partial observability} 
Many real-world reinforcement learning (RL) domains, such as robotics and autonomous navigation, involve physical systems governed by continuous-time dynamics where the agent only receives partial, noisy, sensor readings. We formalize this setting as a continuous dynamical system described by a differential equation and an observation equation: 
\begin{align}
\dot{x}(t) &= f(x(t)) + g(x(t)) u(t) \label{eq:trueODE_x}\\ 
y(t) &= \psi(x(t)) + v(t) \label{eq:trueODE_y} 
\end{align}
where $x \in \mathcal{X} \subseteq \mathds{R}^{d_x}$ represents the true underlying system state, $u(t) \in \mathcal{U} \subseteq \mathds{R}^{d_u}$ is the control input applied by the agent, and $y(t) \in \mathcal{O} \subseteq \mathds{R}^{d_y}$ is the noise-corrupted observation with sensor noise $v(t)$.   
We assume the true drift dynamics $f$ and control-influence matrix $g$ are continuous and differentiable with respect to $x$. In RL, these transition dynamics are unknown to the agent and can only be sampled through interaction.

The optimal control minimizes the cost objective:
\begin{equation}
\begin{aligned}
    J({u}(\cdot)) = \mathds{E}_{v}\left[\Phi (x(t_f)) + \int_{t_0}^{t_f} \Big( q(x(\tau)) + c(u(\tau)) \Big) \dd \tau \right] ,\label{eq:objective}
\end{aligned}
\end{equation}
with terminal cost $\Phi$, state cost $q$, and control cost $c$, all assumed differentiable and evaluable empirically though not known analytically. We write $\mathcal{L}(x,u)=q(x)+c(u)$ for the running cost. 
In RL the objective is commonly discounted, weighting the running cost by $e^{-\rho(t-t_0)}$, the continuous-time counterpart of $\gamma^t$ with $\gamma = e^{-\rho\Delta}$ for step size $\Delta$.

\subsection{Recurrent Policies}
A memory-less policy $\pi_{\theta} \colon \mathcal{O} \to \mathcal{U} $ considers only (the observation of) the current state to decide what action to take, such that $u(t) = \pi_{\theta}(y(t))$ where $\theta \in \Theta$ are the policy parameters. 
Because the true state $x(t)$ cannot be fully inferred from a single observation $y(t)$ under partial observability, this memory-less controller is suboptimal in general. This necessitates a policy with memory, which maintains a latent state $h(t) \in \mathds{R}^{d_h}$ that evolves as a dynamical process coupled to the environment, where $d_h$ is the dimensionality of the hidden state: 
\begin{align}
    \dot{h}(t) = \phi_\theta(y(t), h(t)) , \qquad 
    u(t) = \pi_\theta(h(t)). \label{eq:rnn}
\end{align}
A standard reading of $h$ is as a filter that compresses a sequence of partial observations $y_{0:t}$ and actions $u_{0:t-1}$ into a belief $b_t = p(x_t \mid y_{0:t}, u_{0:t-1})$. This belief is a sufficient statistic for control \citep{aastrom1965optimal, kaelbling1998planning}, and a recurrent policy is understood as approximating it.

Gated cells such as LSTM~\citep{hochreiter1997long} and GRU~\citep{chung2014empirical} are common choices for $\phi_{\theta}$ in deep RL~\citep{wierstra2010recurrent}. Minimal variants such as minGRU and minLSTM~\citep{feng2024rnn} remove the hidden-state dependence from the gates, restoring an affine recurrence. A second family of recurrent cells are structured state space models (SSMs)~\citep{gu2021efficiently, smith2022simplified}, which model their hidden states via a linear recurrence:
\begin{equation}
    \dot{h}(t) = A(t) h(t) + B(t) y(t), \qquad 
    u(t) = C h(t)
\end{equation}
discretized with step $\Delta$ to $h_t = \bar{A} h_{t-1} + \bar{B} y_t$. 
The affine structure of SSMs admits parallel scans over the sequence, enabling credit assignment over long horizons. Selective SSMs such as S6~\citep{gu2023mamba} retain this structure while conditioning $(\bar{A}, \bar{B})$ on the current input.


\section{Recurrent Policies as Optimal Dynamic Processes}
Optimizing a policy for the cost in Eq.~(\ref{eq:objective}) is governed by the structure of the underlying optimal control problem. Whereas the Hamilton-Jacobi-Bellman (HJB) equation characterizes optimality globally across the state space,  PMP characterizes it locally along a trajectory~\citep{pontryagin1987, kirk2004optimal, bryson1975applied}. The trajectory-local view is a natural fit for a recurrent policy, as its hidden state also evolves along the trajectory it controls.

\subsection{Pontryagin Minimum Principle and Hamiltonian Systems}\label{sec:3.1}
To study optimality along a trajectory, we consider the control-Hamiltonian, which augments the running cost with the dynamics weighted by a multiplier $\lambda$:
\begin{equation}
    \mathcal{H}(x, \lambda, u) \coloneqq \mathcal{L}(x,u) + \lambda^{\top} \dot{x} = q(x) + c(u) + \lambda^{\top} \bigl(f(x) + g(x)u\bigr).
\end{equation}
Here, $\lambda \in \R^{d_x}$ is the \textit{co-state}: a continuous-time Lagrange multiplier enforcing the dynamics, which measures the sensitivity of the optimal cost-to-go to the current state. Wherever $V^\star$ is differentiable, it coincides with the value gradient along optimal trajectories, $\lambda(t) = \nabla V^\star(x(t))$~\citep{vinter1986costate}.

\begin{theorem}[Optimal Hamiltonian System] \label{thrm:hamiltonian*}
Let $u^{\star}(t)$ be an optimal solution for the control problem. The optimal trajectory $x^{\star}(t)$ and its corresponding co-state $\lambda^{\star}(t)$ must satisfy the following two-point boundary value problem (TPBVP) for $t \in [t_0, t_f]$:
\begin{equation}
    \begin{split}
        \dot{x}^{\star} &= \nabla_\lambda \mathcal{H} = f(x^{\star}) + g(x^{\star}) u^{\star} \quad \quad \quad \textup{s.t.} \quad   x^{\star} (t_0) = x_0,\\ 
    \dot{\lambda}^{\star} &= -\nabla_x \mathcal{H} = -\nabla_x q - \left(\nabla_x \dot{x}^{\star}\right) \lambda^{\star} \quad \textup{s.t.} \quad  \lambda^{\star} (t_f) = \nabla_x \Phi (x^{\star} (t_f)),\\
    u^{\star} &= \argmin_{u\in\mathcal{U}} \mathcal{H}(x^\star, \lambda^\star, u), 
\end{split}\label{eq:optimalhamiltoniansystem}
\end{equation} where the notation $(\nabla_y f)_{i,j} \coloneqq \partial f_j / \partial y_i$ denotes Jacobians, such that the gradient of the scalar Hamiltonian produces the column vector $\nabla_x \mathcal{H} = ( \frac{\partial \mathcal{H}}{\partial x_1}, \ldots, \frac{\partial \mathcal{H}}{\partial x_{d_x}} )^{\top}$.\footnote{For the discounted infinite-horizon objective, the current-value co-state is $\dot\lambda^\star = \rho\lambda^\star - \nabla_x q - (\nabla_x\dot x^\star)\lambda^\star$. }
\end{theorem}
\textit{Proof:} The derivation follows standard extremization of the cost functional using variational calculus. A full derivation is provided in Appendix~\ref{app:optimalhamiltoniansystem}.

Two structural properties of Eq.~(\ref{eq:optimalhamiltoniansystem}) are particularly important for what is to follow. First, the co-state equation is affine in $\lambda$, meaning that non-linearity in the environment does not require non-linearity in the co-state. Second, because the action $u$ enters the control-Hamiltonian $\mathcal{H}$ linearly, the action that minimizes this Hamiltonian depends on the co-state only through the projection $g(x)^{\top}\lambda$. 

\begin{figure}[t]
\centering
\resizebox{\linewidth}{!}{
\begin{tikzpicture}[scale=1., every node/.style={align=center, font=\normalsize}]
    
    \tikzstyle{boxstyle} = [text width=7.0cm, inner sep=2mm, align=center]

    \node[boxstyle, draw=statecolor, fill=statecolor!10] (left) at (-3.9, 0) {
        \textbf{Optimal Hamiltonian system} \\[0.2em]
        \begin{minipage}{6.8cm}
        \begin{equation*}
        \setlength{\abovedisplayskip}{0pt}
        \setlength{\belowdisplayskip}{0pt}
        \begin{aligned}
            \dot{x}^{\star} &= f(x^{\star}) + g(x^{\star}) u^{\star},  \\
            \dot{\lambda}^{\star} &= -\textcolor{blue}{\nabla_x q} - \textcolor{teal}{\left(\nabla_x \dot{x}^{\star}\right) }\lambda^{\star},   \\ 
            u^{\star} &= \argmin_{u\in\mathcal{U}}\{ c(u) + u^{\top} \textcolor{olive}{g(x^{\star})^{\top}} \lambda^{\star} \},   \\
        \end{aligned}
        \end{equation*}
        \end{minipage}
    };

    \node[boxstyle, draw=costatecolor, fill=costatecolor!10] (right) at (3.9, 0) { 
        \textbf{Co-state Policy} \\[0.2em]
        \begin{minipage}{6.8cm}
        \begin{equation*}
        \setlength{\abovedisplayskip}{0pt}
        \setlength{\belowdisplayskip}{0pt}
        \begin{aligned}
            \dot{x}^{\star} &= f(x^{\star}) + g(x^{\star}) u^{\star}, \quad 
            y^{\star} = \psi(x^{\star} ) + v,  \\
            \dot{h}^{\star} &= -\textcolor{blue}{b_{\theta}(y^{\star})} - \textcolor{teal}{ F_{\theta}(y^{\star},u^{\star})} h^{\star},   \\
            u^{\star} &= \argmin_{u\in\mathcal{U}} \{ c(u) + u^{\top} \textcolor{olive}{G_{\theta}(y^{\star})} h^{\star} \}, \\
        \end{aligned}
        \end{equation*}
        \end{minipage}
    };

\end{tikzpicture}
} 
\caption{\textbf{Structural parallel between the optimal Hamiltonian system (left) and the CP (right)}: an optimal CP represents the co-state as a linear embedding. This results in a policy that follows an optimal dynamic control law via a latent continuous-time process. Color indicates direct correspondence under full observability (\ref{sec:3.2}) and indirect correspondence under partial observability (\ref{sec:3.3}).}
\label{fig:pmp-ctrnn-parallel}
\end{figure}
\subsection{Co-state Policies} \label{sec:3.2}
The PMP describes optimality through two coupled processes: the physical state $x$, and the co-state $\lambda$, with the control readout from the latter. We claim that this is the structure of a recurrent policy: the hidden state $h$ plays the role of $\lambda$, and the policy readout that of the Hamiltonian minimization.
\begin{definition}[Co-state Policy] \label{def:CP*}
A co-state policy (CP) is a recurrent policy of the form
\begin{equation}
    \dot{h} = -b_{\theta}(y) - F_{\theta}(y,u) h,   \qquad 
    u = \argmin_{u\in \mathcal{U}} \{ c(u) + u^{\top} G_{\theta}(y) h \},
\end{equation}where $h \in \R^{d_h}$ is the memory state, and $b_\theta, F_\theta, G_\theta$ may depend arbitrarily on the observation and control but not on $h$. For $c(u) = \tfrac12 u^{\top} R u$ with $R \succ 0$ and $\mathcal{U} = \R^{d_u}$, the readout has the closed form $u = -R^{-1} G_{\theta}(y)\, h$. 
\end{definition} 
The defining property is that the recurrence is affine in $h$ for a given action. Next, we consider when such a network reproduces the optimal co-state exactly. Because recurrent policies are higher-dimensional ($d_h \geq d_x$), we state it through a fixed embedding $T$. 

\begin{proposition}[Co-state Representation]\label{prop:representation}
Let $c(u) = \tfrac12 u^\top R u$ with $R \succ 0$, assume full observability
($\psi$ invertible, $v = 0$), and let $T \in \R^{d_h \times d_x}$ have full column rank and be fixed along the trajectory. If
\begin{align}\label{eq:threeconditions}
    b_{\theta}(y) = T\,\nabla_x q(x), \qquad
    F_{\theta}(y,u)\,T = T\,\nabla_x \dot{x}, \qquad
    G_{\theta}(y)\,T = g(x)^{\top},
\end{align}
and $h(t_0) = T\lambda^{\star}(t_0)$, then the CP realizes the optimal
co-state dynamics $h(t) = T \lambda^{\star}(t)$ for all $t \in [t_0,t_f]$ and emits $u(t) = u^{\star}(t)$ for all $t \in [t_0,t_f]$. 
\end{proposition}

\textit{Proof:} We verify that $h(t) = T\lambda^{\star}(t)$ satisfies the CP dynamics. Substituting the first two conditions into the CP dynamics gives
$\dot h = -b_\theta(y) - F_\theta(y,u)\,T\lambda^{\star}
    = -T\nabla_x q - T(\nabla_x\dot x)\lambda^{\star}
    = T\dot\lambda^{\star},$
so $h$ and $T\lambda^{\star}$ follow the same dynamics. The relation $h = T\lambda^{\star}$ is therefore invariant, and $h(t_0) = T\lambda^{\star}(t_0)$ places the trajectory on it for all $t$. The readout then returns $u = -R^{-1}G_\theta(y)\,T\lambda^{\star} = -R^{-1}g(x)^{\top}\lambda^{\star} = u^{\star}$
by the third condition. \hfill$\square$

This makes the analogy between hidden states and co-states precise: a CP whose operators represent the Hamiltonian terms is an optimal controller, with $h$ a linear image of $\lambda^{\star}$. We summarize this in Fig.~\ref{fig:pmp-ctrnn-parallel}.
While Theorem~\ref{thrm:hamiltonian*} describes a TPBVP that fixes $\lambda$ at the terminal time, a causal recurrent policy must go forward in time and cannot satisfy the end-point boundary directly. However, along optimal trajectories the co-state admits a feedback form, $\lambda^{\star}(t) = \nabla_x V^{\star}(x^{\star}(t),t)$, so it is determined by the current state and time. As such, a recurrent policy can track the co-state forward in time by realizing this value gradient.

\subsection{Partial Observability: Linear-Quadratic-Gaussian Analysis}\label{sec:3.3}
Whereas Prop.~\ref{prop:representation} assumes that the network observes the true state, recurrent policies are deployed to resolve partial observability. We analyze partial observability in the linear-quadratic-Gaussian setting, where the compression of the sequence of observations and actions into a sufficient statistic is exact. Let $\hat{x}$ be the steady-state Kalman estimate and $P$ the solution to the Riccati equation, so that $\lambda^\star = Px$. Setting $h = P\hat{x}$, the CP is written as:
\begin{equation}
    \dot{h} = P\dot{\hat{x}}
    = \underbrace{P L y}_{-b_{\theta}(y)}
    + \underbrace{P\left(A - LC - BR^{-1}B^{\top}P\right)P^{-1}}_{-F_{\theta}}
      h ,
    \label{eq:riccati_diff}
\end{equation} 
where $L$ is the Kalman gain (Appendix~\ref{app:LQG}). This means that the CP operators carry a belief-state reading: $b_\theta$ applies the Kalman innovation in co-state coordinates, while $F_{\theta}$ combines the environment dynamics with the observer correction and the control feedback. This makes CPs applicable to the partially observed settings where recurrent policies are used in practice. 

In non-linear settings, local linearization can approximate the optimal solution via time-varying Jacobians $A(t) = (\nabla_x \dot{x})^{\top}$ and $B(t) = g(x(t))$, following the standard construction of iterative LQG and the extended Kalman filter~\citep{todorov2005generalized}. Because the Riccati solution $P(t)$ and observer gain $L(t)$ now evolve along the trajectory, differentiating the co-state mapping $h(t) = P(t)\hat{x}(t)$ introduces an additional $\dot{P}\hat{x}$ term. Substituting the continuous-time Riccati equation for $\dot{P}$ gives (suppressing $t$ for readability):
\begin{equation}
    \dot{h} = P\dot{\hat{x}} + \dot{P} \hat{x}= \underbrace{P L y}_{-b_{\theta}(y)}
    - \underbrace{\big( A^{\top} + QP^{-1} + PLCP^{-1} \big)}_{F_{\theta}} h .
    \label{eq:tv_riccati_diff}
\end{equation}
This preserves the affine structure of the recurrence: $F_{\theta}$ incorporates the fully observable adjoint dynamics ($A^{\top}$ and $QP^{-1}$) as well as the partial-observability correction ($PLCP^{-1}$). As such, estimation and control are not two separate stages, but enter the same coordinates additively.

\subsection{The readout as Hamiltonian minimization}
In standard actor architectures, the readout that maps the hidden state to the action is chosen by convention. In CPs, it can be chosen by minimizing the control-Hamiltonian, which depends on the co-state and is captured by the switching function $\sigma(t) \coloneqq -G_{\theta}(y) h(t)$ ($=-g(x,t)^{\top} \lambda(t)$ on the manifold). The form of the policy readout is then determined by the control cost (illustrated in Fig.~\ref{fig:controllaws}).

\section{The Control-Hamiltonian Structure of Recurrent Networks}
 
\subsection{Which recurrent cells are co-state policies?}\label{sec:4.1}

\begin{table}[t]
\centering
\begin{threeparttable}
\caption{\textbf{Recurrent cells as co-state policies.} The components of each cell (single-layer) are written as Eq.~(\ref{eq:cp_discrete}). Discrete-time cells use $\Delta_t = 1$; for zero-order-hold cells $b_\theta$ holds up to input scaling.}
\label{tab:membership}
\footnotesize
\begin{tabular*}{\textwidth}{@{\extracolsep{\fill}} l c c c c c @{}}
\toprule
\textbf{Cell} & \textbf{CP} & $\bm{\Delta_t}$ & $\bm{b_\theta}$ & $\bm{F_\theta}$ & $\bm{G_{\theta}}$ \\
\midrule
LSTM~\citeyearpar{hochreiter1997long}, GRU~\citeyearpar{chung2014empirical}
  & \xmark & \multicolumn{3}{c}{non-linear gates on $h_{t-1}$} & \nomark \\
minGRU~\citeyearpar{feng2024rnn}
  & \cmark & $1$ & $-\sigma(W_z y) \odot \tilde h(y)$ & $\mathrm{diag}\big(\mathrm{softplus}(W_z y)\big)$ & $I$ \\
minLSTM~\citeyearpar{feng2024rnn}
  & \cmark & $1$ & $- {i_t}/{(f_t + i_t)} \odot \tilde h(y)$ & $\mathrm{diag}\big(\log(1 + i_t/f_t)\big)$ & $I$ \\
\midrule
S4D~\citeyearpar{gu2022parameterization}, S5~\citeyearpar{smith2022simplified}
  & \cmark & $\Delta$ & $-B\,y$ & $\mathrm{diag}(\nu - i\omega)$ & $C$ \\
LRU~\citeyearpar{orvieto2023resurrecting}
  & \cmark & $1$ & $-\gamma \odot B\,y$ & $\mathrm{diag}(\nu - i\omega)$ & $C$ \\
\midrule
S6~\citeyearpar{gu2023mamba}
  & \cmark & $\Delta_t(y)$ & $-B_t(y)\,y^{\top}$ & $\mathrm{diag}(\nu)$ & $C_t(y)$ \\
Mamba-2 (SSD)~\citeyearpar{dao2024transformers}
  & \cmark & $\Delta_t(y)$ & $-B_t(y)\,y^{\top}$ & $\nu\,I$ & $C_t(y)$ \\
DeltaNet~\citeyearpar{schlag2021linear}
  & \cmark & $1$ & $-\beta_t\, v_t k_t^{\top}$ & $-\log(1-\beta_t)\, k_t k_t^{\top}$ & $q_t(y)$ \\
\bottomrule
\end{tabular*}
\end{threeparttable}
\end{table}

The CP class (Definition~\ref{def:CP*}) asks that the recurrence $h$ enters its own update affinely. After discretizing the continuous-time dynamics with step $\Delta_t$, we can write the following discrete-time CP:
\begin{equation}
    h_t = \bar A_t h_{t-1} + \bar b_t, \qquad
    \bar A_t = \exp\!\big({-}\Delta_t F_\theta(y_t,u_t)\big), \quad
    \bar b_t = -\Delta_t\, b_\theta(y_t).
    \label{eq:cp_discrete}
\end{equation}
To decide membership, we determine whether a recurrent cell's transition can be written in this form.

\textbf{Gated recurrent family.} GRUs and LSTMs contain non-linear updates in the hidden state. The GRU update $h_t = (1-z_t)\odot h_{t-1} + z_t \odot \tanh(W_h y_t + U_h(r_t \odot h_{t-1}))$ is affine in neither term, while LSTM is affine in its cell state $c_t$ but not in its gates, which are non-linear functions of $h_{t-1}$. As such, neither GRU nor LSTM fall inside the CP class. Their minimal variants minGRU and minLSTM~\citep{feng2024rnn}, however, remove the hidden-state dependence from the gates, leaving $h_t = (1-z_t(y))\odot h_{t-1} + z_t(y)\odot W_h y_t$ (equivalent to Eq.~(\ref{eq:cp_discrete}) with $\bar A_t = 1 - z_t(y)$). Therefore, these are CPs. 

\textbf{Linear RNNs and state-space models.} These satisfy
Eq.~(\ref{eq:cp_discrete}) by construction. S4D~\citep{gu2022parameterization}, S5~\citep{smith2022simplified} and the LRU~\citep{orvieto2023resurrecting} use a fixed diagonal $\bar A$; selective cells such as S6~\citep{gu2023mamba} and
Mamba-2~\citep{dao2024transformers} condition $\Delta_t$, $\bar B_t$ and the readout on the observation, making them time-variant. 
Interestingly, the condition that makes a cell a CP is the same one that makes it parallelizable: an affine recurrence composes associatively and can be evaluated using scans. These conditions align because the co-state equation is itself affine in $\lambda$. 
We identify the components $b_\theta, F_\theta, G_\theta$ of each cell for various recurrent cells in Table~\ref{tab:membership}.

\subsection{Supervising the co-state}\label{sec:loss}
A suitable architecture alone does not guarantee that co-state representations naturally emerge under reward maximization, however. One can use the identity $\lambda^{\star} = \nabla_x V^{\star}$ to steer it towards latent representations that align with co-states. In the actor-critic setting, the critic already estimates the value function and hence its gradient. By taking the critic's gradient as a proxy for the value gradient, we construct a hidden state target $\hat{\lambda}_t = \nabla_{y_t} V_\psi(y_t, h^{c}_{t-1})$. The critic's hidden state $h^{c}_{t-1}$ is detached, since the goal is for the actor's hidden state to mirror the critic's gradient instead of the other way around.  Mirroring Prop.~\ref{prop:representation}, we align the two through a linear embedding:
\begin{equation}
    \mathcal{L}_{\text{co-state}}(\theta, W)
    = \mathds{E}_t \left[ \; \big\| W h_t -
      \texttt{sg}\big[\hat\lambda_t\big] \big\|^2 \; \right],
    \label{eq:costate_loss}
\end{equation}
where $h_t$ is the recurrent state, $W \in \R^{d_y \times d_h}$ is a learned projection matrix and $\texttt{sg}$ denotes the stop gradient operator. Here, $\nabla_{y_t} V_{\psi}$ is the observable proxy for $\nabla_x V^{\star}$. Supervising a value gradient makes this loss adjacent to the field of dual heuristic programming~\citep{werbos1992approximate}. Since the loss objective is based purely on the actor and the critic, it can also be applied to recurrent policies outside the CP class.

\section{Experiments}
Our experiments address two questions: $i)$ do hidden states of reward-maximizing policies embed optimal co-states, beyond state estimation? and $ii)$ does supervising the hidden states with the value gradient of the critic improve this embedding as well as the return?

Throughout, we use a recurrent actor to select actions and a recurrent critic to estimate returns. We optimize both with recurrent proximal policy optimization (PPO)~\citep{schulman2017ppo} and generalized advantage estimation (GAE)~\citep{schulman2015high}, computing gradients via back-propagation through time (BPTT). Our implementation extends RSL-RL~\citep{schwarke2025rsl}; all code and training scripts are available via \url{https://github.com/DavidLeeftink/anonymousfornow}.

\subsection{Do hidden states encode co-states?}\label{sec:5.1}

\begin{table}[t]
    \centering
    \caption{\textbf{Linear decodability of optimal co-states on FitzHugh--Nagumo.} We train several policies with a linear readout layer across 10 seeds (mean $\pm$ std), and evaluate them from the same initial condition as the optimal state and co-state trajectories. The value gradient loss is applied with coefficient $c_{\text{co-state}} = 0.2$. The held-out trajectories are fitted across all time steps of the window $[0.5, 10)$ with $d_h=64$ and $\Delta_t = 0.02$ without truncated BPTT. $\Delta_{\text{ceil}}$ is the difference to the linear ceiling. Across the optimal trajectories, this is 0.870; we recompute it across the states visited by each policy in the rollouts. Bolded values indicate the $\Delta$ is more than two standard deviations above the linear ceiling. }
    \label{tab:decoding}
    \small
    \begin{tabular*}{\textwidth}{@{\extracolsep{\fill}} l ccc ccc @{}}
    \toprule
     &  \multicolumn{3}{c}{$R^2$: \textbf{PPO (reward only)}}
         & \multicolumn{3}{c}{$R^2$: \textbf{PPO} $+\,\bm{\mathcal{L}_{\text{co-state}}}$} \\
    \cmidrule(lr){2-4}\cmidrule(lr){5-7}
    \textbf{Cell}
     & $h \to \lambda^\star$ & $\Delta_{\text{ceil}}$ & $h \to x^{\star}$
     & $h \to \lambda^\star$ & $\Delta_{\text{ceil}}$ & $h \to x^{\star}$ \\
    \midrule
    LSTM    & $.727{\pm}.025$ & $-.140$ & $.810{\pm}.018$
            & $.932{\pm}.024$ & $\mathbf{+.066}$ & $.971{\pm}.018$ \\
    GRU     & $.870{\pm}.025$ & $+.003$ & $.929{\pm}.018$
            & $.950{\pm}.004$ & $\mathbf{+.088}$ & $.985{\pm}.003$ \\
    \midrule
    minGRU  & $.885{\pm}.014$ & $+.024$ & $.960{\pm}.010$
            & $.915{\pm}.012$ & $\mathbf{+.055}$ & $.971{\pm}.007$ \\
    minLSTM & $.864{\pm}.020$ & $+.003$ & $.953{\pm}.012$
            & $.920{\pm}.010$ & $\mathbf{+.060}$ & $.974{\pm}.004$ \\
    LRU     & $.865{\pm}.024$ & $+.008$ & $.947{\pm}.016$
            & $.867{\pm}.025$ & $+.033$ & $.932{\pm}.031$ \\
    S6      & $.909{\pm}.003$ & $\mathbf{+.049}$ & $.959{\pm}.006$
            & $.920{\pm}.004$ & $\mathbf{+.060}$ & $.974{\pm}.003$ \\
    \bottomrule
    \end{tabular*}
\end{table}

To determine whether hidden states truly embed value gradients, we distinguish two hypotheses: either $h$ strictly estimates the unobserved state, or it tracks the co-state $\lambda$ explicitly. 
We evaluate these hypotheses on the non-linear FitzHugh-Nagumo system, with the objective to stabilize the system's unforced limit cycle to the origin from partial voltage observations (Appendix~\ref{app:experimentdetails}). Because the system's control influence function $g(x)=( R_{\text{sys}}, 0)^{\top}$ does not vary with the state, the third condition of Prop.~\ref{prop:representation} is met by a constant $G_\theta$. Consequently, the CP readout is linear for this task, ensuring that the experiment isolates the recurrence. 

We first generate $m=500$ ground-truth optimal trajectories $\{x^\star(t), \lambda^\star(t)\}$ by solving the exact Hamiltonian two-point boundary value problem (Theorem~\ref{thrm:hamiltonian*}) via indirect single shooting~\citep{diehl2011numerical}. Independently, we train recurrent PPO agents on the identical undiscounted objective, restricting the actor to this linear readout.
Evaluating the trained policies from the same initial conditions allows us to test Prop.~\ref{prop:representation}, which predicts that $h$ is a linear image of the optimal co-state. We fit a single ridge regression model $\hat{\lambda}_t = W h_t$ across all time steps on 80\% of the trajectories, discarding the initial 5\% for the uninformative $h=\mathbf{0}$ initialization, and report the coefficient of determination ($R^2$) on the held-out 20\%. 

An important baseline ceiling is obtained by fitting a linear model from the true optimal state to the optimal co-state. This is the best any policy could achieve if its memory were to perfectly track the state. 
Across the optimal trajectories this gives $R^2(x^{\star} \to \lambda^\star) = 0.870$; we recompute it per seed across the states each policy visits, and report the difference $\Delta_{\text{ceil}}$. We furthermore report how well $h$ tracks the unobserved state $x$ via $R^2(h \to x^\star)$. 

Table~\ref{tab:decoding} shows that membership in the CP class does not by itself produce co-state structure beyond the ceiling. Under reward maximization, S6 exceeds it, while minGRU, minLSTM, GRU and LRU sit at it, and the non-CP LSTM falls well below. Supervising the hidden state with the critic's value gradient lifts every architecture except LRU above the ceiling. The loss is most effective for the cells outside the class: the GRU attains the highest decodability overall ($R^2 = .950, \Delta_{\text{ceil}} = + 0.088)$. Representing the co-states is therefore driven more by the training objective than by the type of recurrence. Improved co-state alignment is accompanied throughout by improved state tracking, consistent with the analysis in Sec.~\ref{sec:3.3}, where the co-state is an image of the state estimate. State estimation is a component of co-state tracking, but the positive $\Delta_{\text{ceil}}$ shows it is not sufficient. 

\subsection{The Effect of the Co-state Coefficient and Readout Layer}
To study the effect of the co-state loss on the mean return, we consider a partially observable variant of the \texttt{CheetahRun} task from the DeepMind Control Suite without velocity observations (-V)~\citep{tassa2018dmc, han2019variational}. We first evaluate the effect of the co-state loss by augmenting the recurrent PPO objective with the co-state objective (Eq.~(\ref{eq:costate_loss})). We restrict this analysis to the GRU cell with a linear readout layer, and consider the co-state coefficients $c=\{0., 0.025, 0.05, 0.1, 0.2\}$. Because the loss is with respect to the critic, the objective has to be balanced with the other PPO objectives whose constants are kept fixed throughout (an overview of training hyperparameters is given in Table~\ref{tab:hyperparameters}). 

The left panels in Fig.~\ref{fig:ablationA} show that the co-state loss increases the mean episode return for all coefficient values. The second left panel shows that a stronger regularization constant indeed decreases the loss as intended. To investigate the effect of the readout layer, another architectural component of the CP class, we consider a linear readout, an MLP readout, and a selective readout in the form of $C_t (y) h$, and train these with and without co-state loss ($c_{\text{co-state}}=0.1$, based on the left panel results). The right panels show that the readout layers result in similar performances when trained without co-state loss.  With the co-state loss, however, all three readouts improve, including the MLP readout outside the CP class.
\begin{figure}[t]
    \centering
    \includegraphics[width=1.\linewidth]{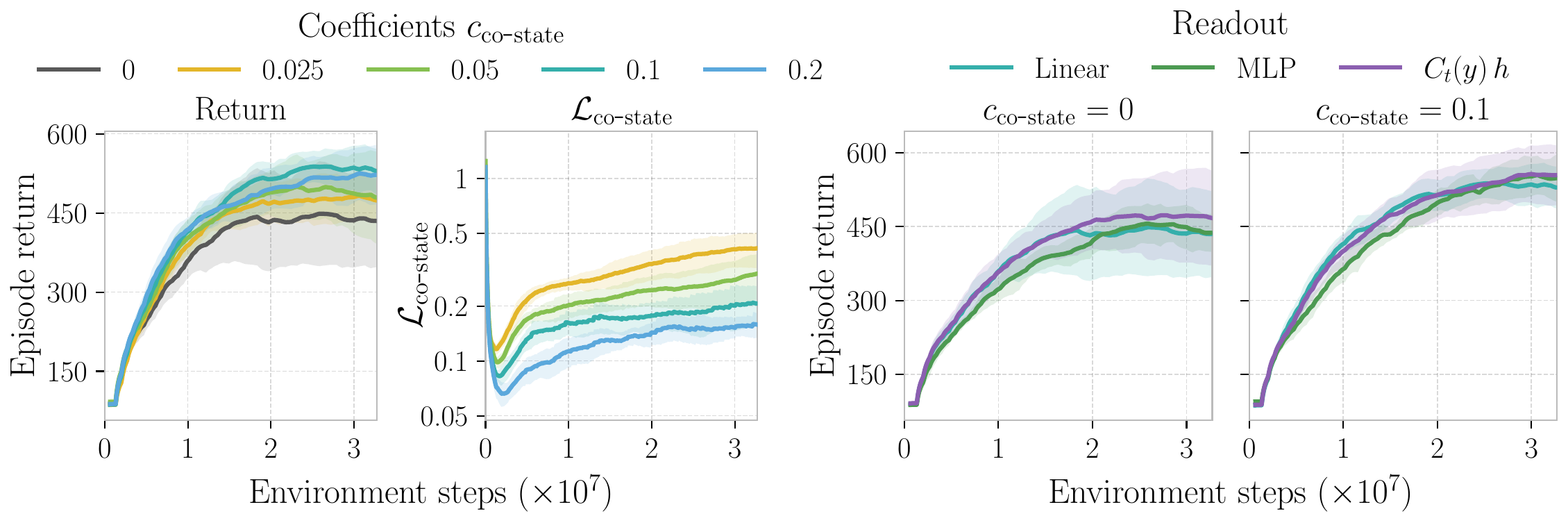}
    \caption{\textbf{Effect of the co-state coefficient and the readout layer on partially observable CheetahRun (velocity removed).} We train a GRU policy with a linear readout layer across 10 seeds (mean $\pm$ std), and evaluate: \textbf{(left)} episode return and co-state loss for $c_{\text{co-state}} \in \{0, 0.025, 0.05, 0.1, 0.2\}$. (\textbf{right}) episode return for the linear, MLP, and selective readout layers, with and without the co-state loss.}
    \label{fig:ablationA}
\end{figure}

\subsection{Humanoid locomotion}
\begin{figure}[t]
    \centering
    \includegraphics[width=1.\linewidth]{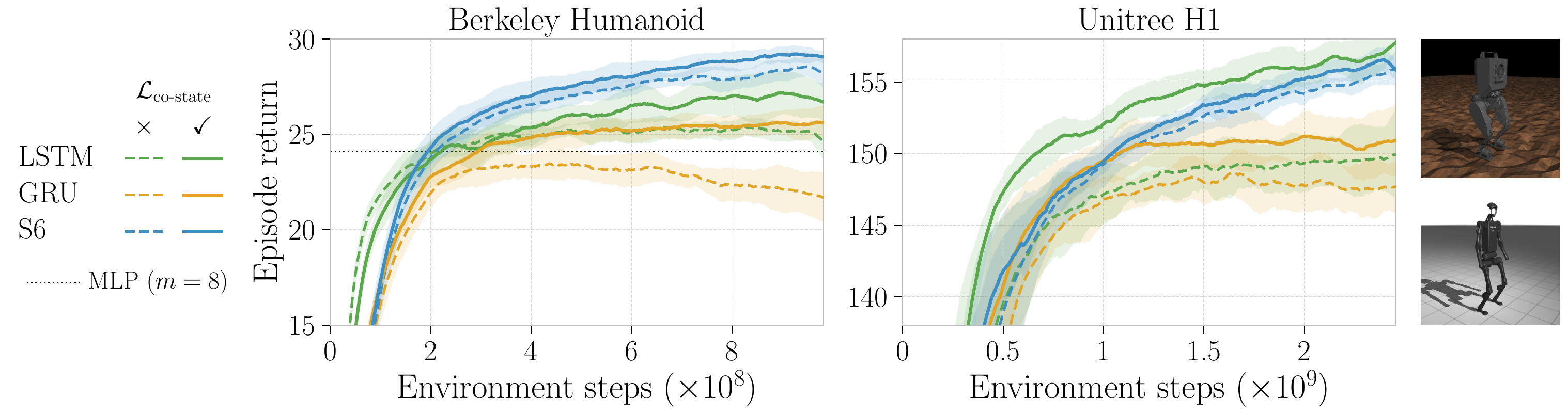}
    \caption{\textbf{Berkeley Humanoid Joystick on Rough Terrain (left) and Unitree H1 Joystick Gait tracking (right).} We train LSTM, GRU and S6 policies across 10 seeds (mean $\pm$ standard deviation), where each cell uses the most beneficial readout for the task (linear for GRU; non-linear for LSTM). Each model is trained \textit{with} (full lines) and \textit{without} (colored dashed lines) the co-state loss with $c_{\text{co-state}}=0.05$. 
    The black dotted line shows the final return of the frame-stacked MLP baseline ($m=8$; Appendix~\ref{app:mlp_baseline}), $24.1 \pm 0.6$ after $4\times10^{8}$ steps on the Berkeley humanoid; on the H1, it attains $99.1 \pm 17.2$ after $8\times10^{8}$ steps, below the plotted range. Both exceed the single-frame results reported by~\citet{zakka2025mujoco}.}
    
    \label{fig:results_humanoids}
\end{figure}

We next evaluate the co-state objective on two partially observable humanoid locomotion tasks: the Berkeley humanoid ($d_y=52, d_u=12$), which requires tracking joystick velocity commands over rough terrain, and the Unitree H1 ($d_y=113, d_u=19$), which navigates flat terrain with a joint objective of velocity and gait tracking~\citep{zakka2025mujoco}. Both environments feature strong partial observability and contact discontinuities. As a non-recurrent baseline, we use an MLP with frame stacking ($m=8$), tuned over stack length and learning rate (Appendix~\ref{app:mlp_baseline}). We compare GRU (linear readout), LSTM (MLP readout), and S6 cells, each trained with and without the co-state loss ($c_{\text{co-state}}=0.05$, selected in Appendix~\ref{app:costate_coef}).

Fig.~\ref{fig:results_humanoids} shows that the co-state objective improves the mean episode return for all cell types on both tasks. The GRU attains the lowest return on both tasks and suffers from long-term training instability on the Berkeley robot; the co-state loss improves its return and resolves this instability. On the H1, the LSTM gains most from the co-state loss, achieving the highest reward on this task and improving substantially on the Berkeley task. Finally, the S6 cell attains the highest return of all baseline models prior to applying the co-state loss, but improves only marginally with it. The architectures that gain most in return are also those whose co-state alignment improves most in Sec.~\ref{sec:5.1}. 

\section{Related work}
\textbf{Recurrent policies and hidden state representations. }
Recurrent policies have been used effectively in deep RL for partially observable environments and remain a strong baseline for POMDPs~\citep{wierstra2010recurrent, hausknecht2015deep, ni2021recurrent}. Their memory states are commonly read as a
belief over the unobserved state, justified by the sufficiency of the belief for control~\citep{aastrom1965optimal, kaelbling1998planning, bertsekas1995dynamic}. We read the same memory through optimal control, where the hidden state carries the co-state of the minimum principle and the readout minimizes the control-Hamiltonian. These two readings are nested rather than opposed: the belief-state sufficiency establishes what information the memory may use, while the minimum principle identifies what part the action requires. This is best seen in the linear-quadratic-Gaussian case (Sec.~\ref{sec:3.3}), where the observer correction enters the same co-state coordinates as the adjoint dynamics. 

\textbf{Pontryagin's Minimum Principle in Deep RL.} While the Hamilton-Jacobi-Bellman (HJB) equation forms the theoretical backbone of value-based RL, its trajectory-centric counterpart, the minimum principle, has played a smaller role. The closest precedent is adaptive dynamic programming, where dual heuristic programming trains a critic to output the value gradient directly and derives the control from this~\citep{werbos1992approximate}. This is model-based and uses the value gradient as a training signal for a memoryless actor. More recently, PMP has been used in model-based paradigms, where the Pontryagin conditions are differentiated to obtain gradients through an optimal trajectory~\citep{jin2020pontryagin}, and for planning under uncertainty~\citep{leeftink2025probabilistic}. The CP framework instead is model-free and uses the minimum principle to characterize what recurrent memory represents. 

\textbf{Modern sequence models.} A parallel line of work replaces attention with recurrences that can be evaluated by parallel scans, including linear transformers~\citep{katharopoulos2020transformers}, fast-weight programmers and DeltaNet~\citep{schlag2021linear}, gated linear attention~\citep{yang2025gated} and Mamba-2~\citep{dao2024transformers}. As shown in Sec.~\ref{sec:4.1}, the affine update that enables their parallel scans is exactly what places them in the CP class, giving a control-theoretic motivation for adopting them in reinforcement learning, beyond the computational one. Such recurrences are increasingly used as backbones in vision-language-action (VLA) models~\citep{liu2024robomamba}, whose memory the co-state loss could supervise directly during actor-critic fine-tuning. 

\section{Discussion}
We propose that the hidden state acts as the Pontryagin co-state, while the readout minimizes the control-Hamiltonian. We formalize this through the co-state policy class, which encompasses many modern architectures from linear RNNs to selective state space models. Because the co-state coincides with the value gradient along optimal trajectories, the critic's gradient can directly supervise the actor's memory.
Our experiments show that while co-state policies can realize optimal co-state dynamics, most architectures do not represent them when trained under reward maximization alone. Directly supervising the hidden state with the value gradient improves both co-state representation and mean return.

A limitation of the algorithmic use of the co-state identity $\lambda^{\star} = \nabla_x V^{\star}(x^{\star})$ is that it only holds where $V^{\star}$ is differentiable. While our results show that the co-state objective still helps on contact-rich tasks such as humanoid locomotion, future work is needed to understand the consequences. A second limitation is that while our partial observability analytical results are exact for the linear-quadratic-Gaussian setting and extend to non-linear systems through local linearization, the global non-linear setting is studied only via our empirical results. 

A promising direction of future work is to investigate the role of hidden state initialization. While the standard practice of initializing the hidden state at zero comes with little theoretical motivation, the CP class prescribes that the initial hidden states follow from the value gradient, $h(t_0) = T \nabla_x V^{\star}(x_0)$, which is not available at run time. A different direction is extending the correspondence to stochastic hidden states via the stochastic maximum principle~\citep{peng1990general}. Finally, the connection between hidden states and PMP opens a direction for safe reinforcement learning under partial observability. State constraints enter the control-Hamiltonian through additional multiplier terms~\citep{hartl1995survey}, which recurrent policies of this class could carry in the same memory that represents the co-state.  

Taken together, these results connect the minimum principle to recurrent memory. They offer an account of hidden states as carriers of the optimality structure of the underlying control problem, rather than of the state alone, together with a mechanism for shaping them toward it.

\subsubsection*{Acknowledgments}
This publication is part of the project ROBUST: Trustworthy AI-based
Systems for Sustainable Growth with project number KICH3.LTP.20.006,
which is (partly) financed by the Dutch Research Council (NWO), ASMPT,
and the Dutch Ministry of Economic Affairs and Climate Policy (EZK)
under the program LTP KIC 2020-2023. All content represents the opinion
of the authors, which is not necessarily shared or endorsed by their
respective employers and/or sponsors.

\bibliography{references_compact}
\bibliographystyle{iclr2027_conference}

\appendix

\section*{Appendix}
\section{Optimal Hamiltonian System} \label{app:A}

\subsection{Additional background on co-states and the control-Hamiltonian}
Repeating the main text for completeness, the control-Hamiltonian expresses how the cost-to-go changes over time with the dynamics:
\begin{equation}
    \mathcal{H}(x, \lambda, u) \coloneqq \mathcal{L}(x,u) + \lambda^{\top} \dot{x} = q(x) + c(u) + \lambda^{\top} \bigl(f(x) + g(x)u\bigr).
\end{equation}
Here, $\lambda \in \mathds{R}^{d_x}$ are referred to as co-states, and encode the sensitivity of the value of a state with respect to the dynamics. These follow from solving the constrained optimization problem in Eq. \ref{eq:objective} using the Lagrange multiplier function $\lambda(t)$ for $t \in [t_0, t_f]$. By defining the differential function of the co-states, we obtain a coupled ODE:
\begin{align}
    \dot{x}(t) &= \nabla_\lambda \mathcal{H}(x, \lambda, u) = f(x(t)) + g(x(t))u, \\
    \dot{\lambda}(t) &= -\nabla_x \mathcal{H}(x, \lambda, u) = -\nabla_x q - \left(\nabla_x \dot{x}\right) \lambda(t)
\end{align}
for $t \in [t_0, t_f]$, where the notation $(\nabla_y f)_{i,j} \coloneqq \partial f_j / \partial y_i$ denotes Jacobians, in this case $\nabla_x \mathcal{H} = ( \frac{\partial \mathcal{H}}{\partial x_1}, \ldots, \frac{\partial \mathcal{H}}{\partial x_{d_x}} )^{\top}$. The interaction dynamics between the states and co-states form a coupled dynamical system, which we refer to as the Hamiltonian system.
Since the co-states incorporate the cost function, this representation lends itself well to the inclusion of a notion of optimality. The minimum principle by \citet{pontryagin1987}, for historical reasons also referred to as the maximum principle, states:
\begin{equation} 
    u^{\star} = \argmin_{u\in\mathcal{U}} \mathcal{H}(x^{\star}, \lambda^{\star}, u) \label{eq:minimumprinciple_boundedcontrols}
\end{equation}
This tells us that if the optimal control $u^{\star}(t)$ is applied, producing corresponding optimal trajectories $x^{\star}(t)$ and $\lambda^{\star}(t)$, the Hamiltonian function is minimized at all time points $t \in [t_0, t_f]$. This is a necessary but not a sufficient condition of optimality.

\subsection{Derivation of the optimal Hamiltonian system}\label{app:optimalhamiltoniansystem}
In this section, we provide the derivation of the necessary optimality conditions for the deterministic continuous-time control problem, drawing on standard results from the calculus of variations \citep{bryson1975applied, kirk2004optimal}. For brevity, we consider an unbounded control set $\mathcal{U} = \R^{d_u}$ which we approach in continuous-time by extremizing the Gateaux derivative. For bounded control sets, the proof is significantly more involved and involves the \textit{needle-variation}, originally proposed by~\citet{pontryagin1987}. We refer the reader to Chapter 4 in~\citet{liberzon2011calculus} for the needle-variation proof.

Let $u^{\star}(t)$ be an optimal control that minimizes the cost functional $J(u)$ over the set of admissible controls $\mathcal{U}$. Let the optimal trajectory be $x^{\star}(t)$, with corresponding co-states $\lambda^{\star}(t)$, and let the optimal Hamiltonian be evaluated as $\mathcal{H}^{\star} \coloneqq \mathcal{H}(x^{\star}(t), \lambda^{\star}(t), u^{\star}(t))$. A necessary condition for optimality is that the gradient of the Hamiltonian with respect to the control vanishes.

\textbf{Proof.} 
To enforce the dynamic equality constraint $\dot{x} = f(x) + g(x)u$, we introduce the Lagrange multiplier function $\lambda(t)$ (the co-state) that ensures this constraint for all $t \in [ t_0, t_f]$. Because the dynamic constraint equals zero along any valid trajectory, adding it to the integral does not change the value of $J(u)$:
\begin{equation}
    J(u) = \Phi(x(t_f)) + \int_{t_0}^{t_f} \left[ q(x) + c(u) + \lambda(t)^\top \Big( f(x) + g(x)u - \dot{x} \Big) \right] \dd t.
\end{equation}
By substituting the definition of the control-Hamiltonian, $\mathcal{H}(x, \lambda, u) = q(x) + c(u) + \lambda^\top (f(x) + g(x)u)$, we can rewrite the objective as:
\begin{equation}
    J(u) = \Phi(x(t_f)) + \int_{t_0}^{t_f} \left[ \mathcal{H}(x, \lambda, u) - \lambda(t)^\top \dot{x} \right] \dd t.
\end{equation}

We consider the G\^{a}teaux derivative of this cost functional:
\begin{equation}
    \delta J(u) = \lim_{\epsilon \to 0} \frac{J(u + \epsilon \delta u) - J(u)}{\epsilon},
\end{equation}
where $\delta u(t)$ is an arbitrary control variation defined over $[t_0, t_f]$. The first variation can be obtained by isolating the first-order terms around the optimal solution $u^{\star}(t)$:
\begin{equation}
    J(u^{\star} + \epsilon \delta u) - J(u^{\star}) \approx \delta J|_{u^{\star}}(\delta u)\epsilon,
\end{equation}
where higher-order terms are neglected as $\epsilon \to 0$. Since, by definition, $u^{\star}(t)$ is a minimum for $J$, its first variation must be zero. Note that a variation in the control $\delta u(t)$ naturally induces a corresponding variation in the state trajectory, denoted as $\delta x(t)$.
A well-known result in optimal control theory~\citep[Eq. 5.1.13]{kirk2004optimal} states that to successfully absorb the variations in the state trajectory $\delta x(t)$, the co-state multiplier $\lambda(t)$ must be chosen to satisfy the differential equation:
\begin{equation}
    \dot{\lambda}^{\star}(t) = -\nabla_x \mathcal{H}(x^{\star}, \lambda^{\star}, u^{\star}),
\end{equation}
with the terminal transversality condition $\lambda^{\star}(t_f) = \nabla_x \Phi(x^{\star}(t_f))$.
With the state variations eliminated by this specific choice of optimal co-state dynamics, the remaining first variation of the cost with respect to the control is given by~\citep[Eq. 5.1.14]{kirk2004optimal}:
\begin{equation}
    \delta J|_{u^{\star}}(\delta u) = \int_{t_0}^{t_f} \left[ \nabla_u \mathcal{H}^{\star} \right]^{\top} \delta u(t) \dd t.
\end{equation}
By the fundamental lemma of the calculus of variations, since $\delta u(t)$ is an arbitrary variation, the integral can only equal zero if the integrand itself vanishes identically. It follows that:
\begin{equation}
    \nabla_u \mathcal{H}^{\star} = 0, \quad \text{for all } t \in [t_0, t_f].
\end{equation}

For unconstrained control sets, this stationary condition defines the optimal control law. For bounded, closed control sets $\mathcal{U}$, this local stationary condition is generalized by the PMP to the global minimization of the Hamiltonian across the action space:
\begin{equation}
    u^{\star}(t) = \argmin_{u\in\mathcal{U}} \mathcal{H}(x^{\star}(t), \lambda^{\star}(t), u).
\end{equation}
This establishes the optimality condition for the control sequence.

\section{Linear-Quadratic-Gaussian Analysis}\label{app:LQG}
We consider the LQG setting, with linear dynamics $\dot x = Ax + Bu + w$, observations $y = Cx + v$ with $v \sim \mathcal{N}(0, \Sigma_v)$, $w \sim \mathcal{N}(0, \Sigma_w)$, and quadratic cost $\mathcal{L}(x,u) = \tfrac12(x^{\top}Qx + u^{\top}Ru)$ with $Q \succeq 0$ and $R \succ 0$. Assume $(A,B)$ stabilizable and $(A,C)$ detectable, so that the
steady-state Kalman gain $L$ and the stabilizing solution $P \succ 0$ of the control algebraic Riccati equation $A^{\top}P + PA - PBR^{-1}B^{\top}P + Q = 0$ both exist. 

By the separation principle, the Kalman estimate $\hat x$ is a sufficient statistic for the belief~\citep{athans1971role}, and the optimal control is the full-information law evaluated at it,
\begin{equation} 
    u^{\star} = -R^{-1}B^{\top}P\,\hat{x}, \qquad
    \dot{\hat{x}} = (A - LC)\hat{x} + Bu + Ly .
\end{equation}
The co-state of the deterministic problem satisfies $\lambda^{\star} = Px$, so in the partially observed case the corresponding quantity is $h = P\hat{x}$, an estimate of $\lambda^\star$ itself. This is the CP embedding of Prop.~\ref{prop:representation} with $T = I$.

Differentiating and substituting the control law gives
\begin{equation}
    \dot h = P\dot{\hat{x}}
      = P\left(A - LC - BR^{-1}B^{\top}P\right)\hat{x} + PLy ,
\end{equation}
and rewriting $\hat{x} = P^{-1}h$ leads to Eq.~(\ref{eq:riccati_diff}), with $b_\theta(y) = -PLy$ and $F_\theta = -P(A - LC - BR^{-1}B^{\top}P)P^{-1}$. Both are independent of $h$, so the recurrence is affine in $h$ and the policy is a CP. The readout $u = -R^{-1}B^{\top}P\hat{x} = -R^{-1}B^{\top}h$ is linear in $h$, matching $G_\theta = B^{\top}$ in Definition~\ref{def:CP*}.~$\hfill\square$

\section{Readout Layer Design via Optimal Control Principles} \label{app:readout}
\begin{figure}[!t]
    \centering
    \resizebox{\columnwidth}{!}{%
    \begin{tikzpicture}[>=stealth, every node/.style={transform shape}]

        \def\umax{1.2}  

        \def\xmax{3.6}  

        \def\liney{3.9}




        \draw[thick, black!20] (-1.0, \liney) -- (15.0, \liney);

        \draw[thick, black!15] (4.7, -2.6) -- (4.7, \liney);

        \draw[thick, black!15] (10.5, -2.6) -- (10.5, \liney);




        \begin{scope}[xshift=0cm]

            \node[anchor=south, align=center, text width=5cm] at (1.6, \liney+0.1) {\textbf{1. Quadratic Control}};


            \node[anchor=north, align=center, text width=5cm] at (1.6, \liney-0.1) {
                $\mathcal{L}(x,u) = q(x) + u^\top R u$
            };


            \node[anchor=north, align=center, text width=5cm] at (1.6, 3.15) {

                $u^* = R^{-1} \sigma $ \\ \vspace{10pt}
                where $\sigma = - \frac{1}{2}g(x,t)^{\top} \lambda(t)$

            };

            \def\quadpath{(0, 1.0) to[out=0, in=135] (1.0, 0) to[out=-45, in=180] (2.2, -1.0) to[out=0, in=180] (3.4, 0)}

        \begin{scope}[yshift=-.35cm]
    \def\curvefunc{0.88*\umax * (1 + 1.1*\x) * exp(-0.6*\x) * cos(deg(\x*1.15))} 
    \def\tend{3.5} 
    \def\crossing{1.36} 

    \fill[blue!10] plot[domain=0:\crossing, samples=100] (\x, {\curvefunc}) -- (\crossing,0) -- (0,0) -- cycle;
    \fill[blue!10] plot[domain=\crossing:\tend, samples=100] (\x, {\curvefunc}) -- (\tend,0) -- (\crossing,0) -- cycle;

    \draw[dashed, black!15, thick] (0, \umax) -- (\xmax, \umax);
    \draw[dashed, black!15, thick] (0, -\umax) -- (\xmax, -\umax);
    
    \draw[->, thick, black!80] (-0.2, 0) -- (\xmax, 0) node[right, black] {$t$};
    \draw[->, thick, black!80] (0, -\umax-0.4) -- (0, \umax+0.5) node[left, black] {$u^*(t)$};

    \draw[thick] (-0.1, \umax) -- (0.1, \umax) node[left=2pt] {$u_{\max}$};
    \draw[thick] (-0.1, -\umax) -- (0.1, -\umax) node[left=2pt] {$u_{\min}$};

    \draw[very thick, blue!80, smooth, samples=100] 
        plot[domain=0:\tend] (\x, {\curvefunc});
\end{scope}

        \end{scope}




        \begin{scope}[xshift=5.8cm]

            \node[anchor=south, align=center, text width=5cm] at (1.6, \liney+0.1) {\textbf{2. Time/State (Bang-Bang)}};

            \node[anchor=north, align=center, text width=5cm] at (1.6, \liney-0.1) {
                $\mathcal{L}(x,u) = q(x) \text{ or } 1$
            };
            \node[anchor=north, align=center, text width=5cm] at (1.6, 3.25) {
                           $\begin{aligned}
                    u^* = \begin{cases}
                        \text{sign}(\sigma) u_{\text{max}} & \text{if } |\sigma| > 0\\ 
                        0 & \text{if } \sigma = 0\\ 
                    \end{cases} 
                \end{aligned}
                $
                where $\sigma = - \frac{1}{2} g(x,t)^{\top} \lambda(t)$
            };

            \begin{scope}[yshift=-0.35cm]

                \fill[violet!10] (0, \umax) rectangle (1.3, 0);

                \fill[violet!10] (1.3, -\umax) rectangle (2.6, 0);

                \fill[violet!10] (2.6, \umax) rectangle (3.4, 0);

                \draw[dashed, black!15, thick] (0, \umax) -- (\xmax, \umax);

                \draw[dashed, black!15, thick] (0, -\umax) -- (\xmax, -\umax);

                \draw[->, thick, black!80] (-0.2, 0) -- (\xmax, 0) node[right, black] {$t$};

                \draw[->, thick, black!80] (0, -\umax-0.4) -- (0, \umax+0.5) node[left, black] {$u^*(t)$};

                \draw[thick] (-0.1, \umax) -- (0.1, \umax) node[left=2pt] {$u_{\max}$};

                \draw[thick] (-0.1, -\umax) -- (0.1, -\umax) node[left=2pt] {$u_{\min}$};

                \draw[very thick, violet!80] (0, \umax) -- (1.3, \umax) -- (1.3, -\umax) -- (2.6, -\umax) -- (2.6, \umax) -- (3.4, \umax);

            \end{scope}

        \end{scope}




        \begin{scope}[xshift=11.6cm]

            \node[anchor=south, align=center, text width=5cm] at (1.6, \liney+0.1) {\textbf{3. Fuel (Bang-Off-Bang)}};

            \node[anchor=north, align=center, text width=5cm] at (1.6, \liney-0.1) {
                $\mathcal{L}(x,u) = q(x) + \|u\|$
            };
            \node[anchor=north, align=center, text width=5cm] at (1.6, 3.3) {
            $\begin{aligned}
                    u^* = \begin{cases}
                        \text{sign}(\sigma) u_{\text{max}} & \text{if } |\sigma| > 1\\ 
                        0 & \text{if } |\sigma| < 1 \\ 
                    \end{cases} 
                \end{aligned}
                $
                where $\sigma = - \frac{1}{2} g(x,t)^{\top} \lambda(t)$

            };
            \begin{scope}[yshift=-.35cm]

                \fill[teal!10] (0, \umax) rectangle (1.0, 0);

                \fill[teal!10] (2.2, -\umax) rectangle (2.9, 0);

                \draw[dashed, black!15, thick] (0, \umax) -- (\xmax, \umax);

                \draw[dashed, black!15, thick] (0, -\umax) -- (\xmax, -\umax);

                \draw[->, thick, black!80] (-0.2, 0) -- (\xmax, 0) node[right, black] {$t$};

                \draw[->, thick, black!80] (0, -\umax-0.4) -- (0, \umax+0.5) node[left, black] {$u^*(t)$};

                \draw[thick] (-0.1, \umax) -- (0.1, \umax) node[left=2pt] {$u_{\max}$};

                \draw[thick] (-0.1, -\umax) -- (0.1, -\umax) node[left=2pt] {$u_{\min}$};

                \draw[very thick, teal!80] (0, \umax) -- (1.0, \umax) -- (1.0, 0) -- (2.2, 0) -- (2.2, -\umax) -- (2.9, -\umax) -- (2.9, 0) -- (3.4, 0);

                \node[above, yshift=2pt, teal!80!black] at (1.6, 0) {\textbf{\textit{Coast}}};

            \end{scope}

        \end{scope}

    \end{tikzpicture}
    } 

    \caption{\textbf{Hamiltonian policies under varying cost functions:} Energy-optimal (1), time-optimal (2), and fuel-optimal (3). This can produce either smooth optimal control functions or discontinuous bang-bang controls. }\label{fig:controllaws}

\end{figure} 
In standard deep RL architectures of the actor, a readout layer (or policy head) that maps the final hidden state to an action is typically chosen heuristically. However, because the CP class explicitly aims to align the hidden state with the theoretical Pontryagin co-state ($h \approx \lambda^{\star}$), the optimal readout layer is no longer arbitrary but instead has to minimize the control-Hamiltonian function with respect to the specific running cost function of the environment. 

By defining the optimal switching function as $\sigma(t) = -g(x,t)^{\top} \lambda(t)$ and assuming symmetric actuator limits $\mathcal{U} \in [-u_{\max}, u_{\max}]$ for notational clarity, we can structurally design the read out layer to match different classes of optimal control problems (illustrated in Fig.~\ref{fig:controllaws}):

\textbf{1. Quadratic Control (Smooth Action):}
Environments with a quadratic penalty on control effort, $\mathcal{L}(x,u) = q(x) + u^{\top} R u$ and control-affine system dynamics $\dot{x}=f(x) + g(x) u$, lead to a Hamiltonian that is convex with respect to $u$. Setting the gradient to zero results in the optimal control law: $u^{\star} = R^{-1}\sigma$. In our framework, this corresponds to a standard linear readout layer without bounding activation functions.

\textbf{2. Time-Optimal and State-Only Costs (Bang-Bang Control):}
For many robotics tasks (e.g., minimum-time navigation), the cost consists only of state-dependent terms, meaning control effort is not penalized: $\mathcal{L}(x,u) = q(x)$. Because the Hamiltonian is linear with respect to $u$, and physical actuators possess strict limits $[u_{\min}, u_{\max}]$, the minimum principle dictates that the optimal solution lies exclusively on the boundaries. This results in \textit{bang-bang} control, where the agent switches instantaneously between maximum and minimum effort based on the sign of the switching function:
$$ u^{\star} = \begin{cases} \text{sign}(\sigma) u_{\max} & \text{if } |\sigma| > 0 \\ 0 & \text{if } \sigma = 0 . \end{cases} $$
In the CP framework, a bang-bang optimal controller can be enforced simply by applying a $\text{sign}(\cdot)$ activation function to the read out layer.

\textbf{3. Fuel-Optimal Costs (Bang-Off-Bang Control):}
If the environment penalizes the absolute magnitude of the control effort (e.g., conserving fuel in aerospace applications), an $L_1$ penalty is introduced: $\mathcal{L}(x,u) = q(x) + \|u\|$. The optimal control law develops a deadzone where it is optimal to coast (apply no control effort) when the sensitivity is low:
$$ u^{\star} = \begin{cases} \text{sign}(\sigma) u_{\max} & \text{if } |\sigma| > 1 \\ 0 & \text{if } |\sigma| < 1 .\end{cases} $$

While the empirical evaluations in this work focus on the continuous quadratic case, the interpretation of recurrent hidden states as Pontryagin co-states allows one to design readout layers specific to the properties of the cost or reward function of the environment while leaving the underlying recurrent co-state dynamics unchanged. 

\section{Epistemic Uncertainty and Mean Hamiltonian Minimization} \label{app:epistemic}
An agent attempting to make an intelligent trade-off between exploration and exploitation must focus its exploration on \textit{epistemic} uncertainty, associated with what the agent does not know yet based on limited interaction~\citep{hullermeier2021aleatoric}. Contrary to aleatoric uncertainty, which represents inherent and irreducible randomness, epistemic uncertainty can be reduced by gathering more data. Explicitly aiming to minimize this uncertainty is a key ingredient in active exploration, enabling efficient learning in applications such as robotics~\citep{deisenroth2011pilco,li2025uncertainty} or semiconductor manufacturing~\citep{Leeftink2026laser}. In the model-based RL approach, epistemic uncertainty is often modeled as parametric uncertainty, obtained through Bayesian inference over the environment's dynamics~\citep{chua2018deep}. 

We can formulate this by considering a policy that aims to minimize the expected cost over a posterior distribution of plausible dynamic models:
\begin{align}
    \pi^{\star} &\coloneqq \argmin_{\pi \in \Pi} \mathbb{E}_{\zeta \sim p(\zeta\mid\mathcal{D})}\left[ J(\pi_{\theta}(x_{\zeta}(t)))\right]   \\
    \text{s.t.} \quad \dot{x}_{ \zeta}(t) &= f_{ \zeta}\bigl(x_{ \zeta}(t),u(t), t\bigr),\quad \text{for } t \in [t_0, t_f] \enspace \nonumber\\ 
    x_{ \zeta}(t_0) &= x_0  . \nonumber   
\end{align}
where $J_{\zeta}(\pi)$ denotes the cost of policy $\pi$ under dynamics $f_{\zeta}$, and $p(\zeta\mid\mathcal{D})$ is the posterior distribution given past interaction data $\mathcal{D}$. 
To solve this probabilistic optimization problem, we can draw a finite set of $m$ samples $\zeta_i \sim p(\zeta\mid\mathcal{D})$. These samples form an ensemble of deterministic dynamical systems driven by a shared control input $u(t)$:
\begin{align}
    \dot{x}_{\zeta_i}(t) &= f_{\zeta_i}\bigl(x_{\zeta_i}(t), u(t),t \bigr), \label{eq:xdot_conditiontheta}\\
    \dot{\lambda}_{\zeta_i}(t) &= -\nabla_{x} \mathcal{H} \bigl(x_{\zeta_i}(t),\lambda_{\zeta_i}(t),u(t),t\bigr),\label{eq:lambdadot_conditiontheta} \\
    \text{s.t.} \quad \lambda_{\zeta_i}(t_f) &= \nabla_x \Phi\bigl(x_{\zeta_i}(t_f)\bigr). \label{eq:transversality_conditiontheta}
\end{align}

If $u^{\star}(t)$ is an optimal control sequence for this ensemble, a necessary optimality condition is that it minimizes the mean Hamiltonian~\citep{ross2015riemann, phelps2016optimal, leeftink2025probabilistic}:
\begin{equation} 
    u^{\star}(t) = \argmin_{u\in\mathcal{U}}\mathbb{E}_{\zeta \sim p(\zeta\mid\mathcal{D})}\left[\mathcal{H}\left(x_{\zeta}(t),\lambda_{ \zeta}(t), u,t \right)\right], \label{eq:meanhamiltonian}
\end{equation}

\textbf{Implications for CPs.} 
While the CP considered in this work optimize a deterministic point estimate, Eq.(~\ref{eq:meanhamiltonian}) provides a theoretical foundation for a \textit{Bayesian} CP. Because the theoretical co-state is linked to the value function gradient, $\lambda^{\star} = \nabla_x V^{\star}$, this uncertainty can be captured by an ensemble of $m$ critics. By maintaining parallel hidden states $h_{\zeta_i}(t)$ across these critics, each representing the value function under one dynamics hypothesis, the actor can account for epistemic uncertainty by minimizing the mean of the ensemble Hamiltonians:
\begin{equation}
    u^{\star}(t) \approx \argmin_{u\in\mathcal{U}} \frac{1}{m} \sum_{i=1}^{m} \left\{ c(u) + u^{\top} G_{\theta}(y) h_{\zeta_i}(t) \right\}.
\end{equation}
We leave the empirical validation of this Bayesian ensemble approach to future work.

\section{Additional humanoid experiments}\label{app:humanoid_extra}
\subsection{Frame-stacked MLP baseline}\label{app:mlp_baseline}

\begin{figure}[h]
    \centering
    \includegraphics[width=0.8\linewidth]{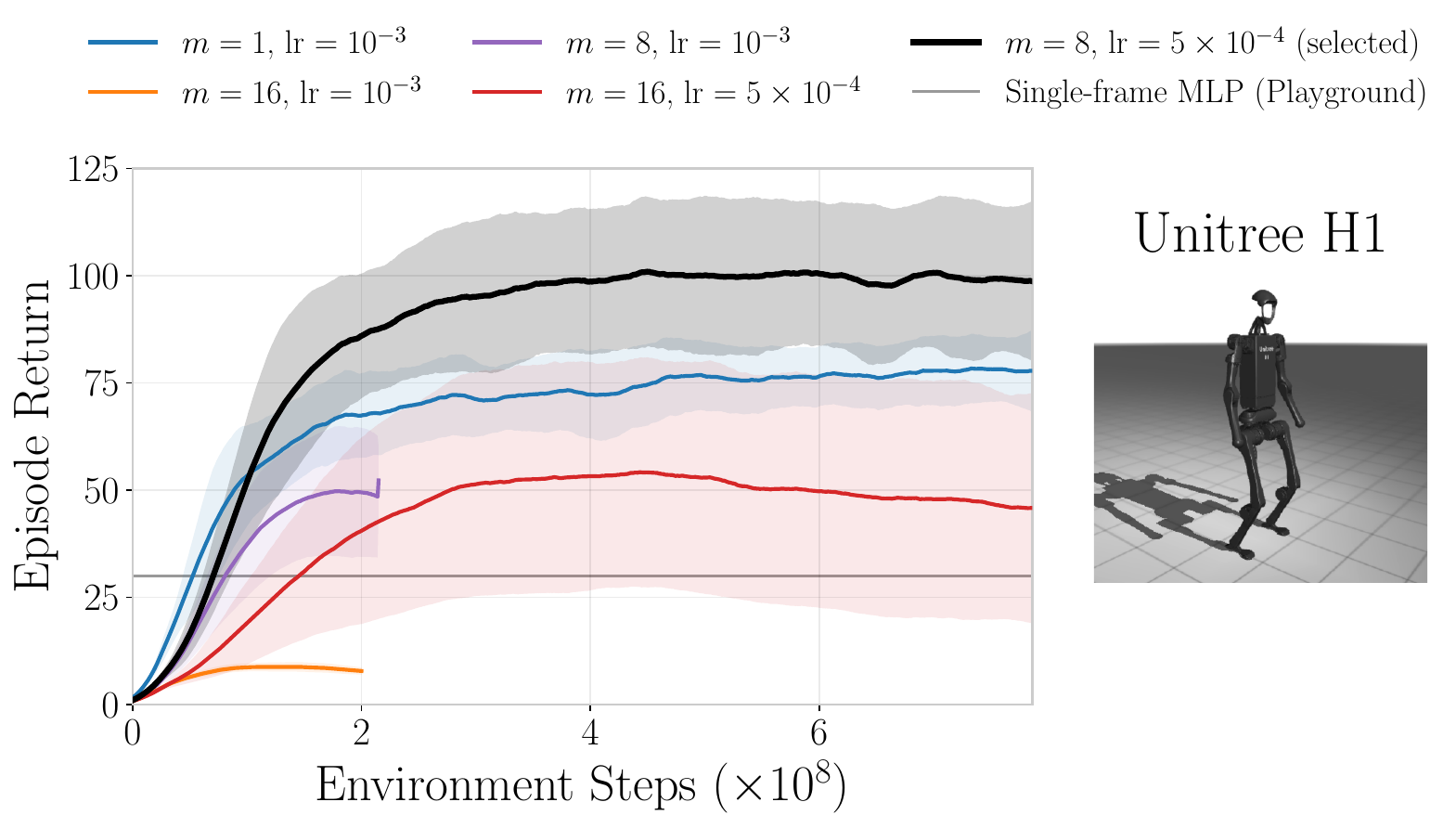}\\[0.8em]
    \includegraphics[width=0.77\linewidth]{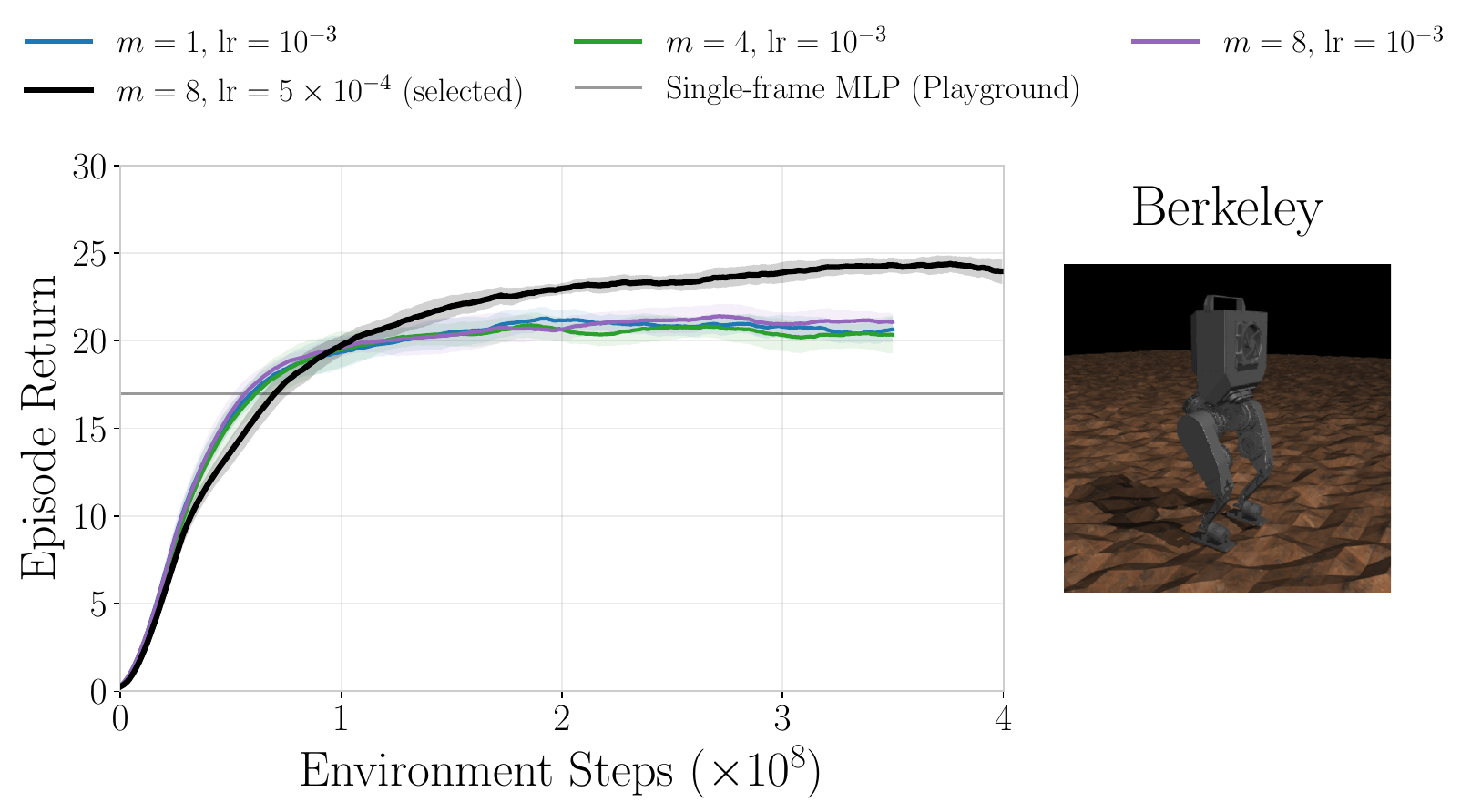}
    \caption{\textbf{Frame-stacked MLP baseline on the Unitree H1 Joystick Gait Tracking (top) and Berkeley Humanoid Joystick Rough Terrain (bottom).} We train MLP policies on the last $m$ observations for different stack lengths and learning rates across 10 seeds (mean $\pm$ std). The grey line marks the single-frame MLP score reported by~\citet{zakka2025mujoco}.     
    On the H1, runs with learning rate $10^{-3}$ for $m \in \{8, 16\}$ were stopped early due to poor performance; on the Berkeley humanoid, runs with learning rate $10^{-3}$ are shown up to $3.5\times10^{8}$ steps, after which not all seeds were logged.
    }
    \label{fig:mlp_sweep}
\end{figure}

As a non-recurrent baseline for the humanoid tasks, we train an MLP policy on a stack of the $m$ most recent observations. We evaluate the stack lengths and learning rates shown in Fig.~\ref{fig:mlp_sweep} ($m \in \{1, 8, 16\}$ on the H1 and $m \in \{1, 4, 8\}$ on the Berkeley humanoid), keeping all other hyperparameters as in Table~\ref{tab:hyperparameters}, and train each for up to $8 \times 10^{8}$ (H1) and $4 \times 10^{8}$ (Berkeley) environment steps.

On the Unitree H1, frame stacking with $m=8$ and learning rate $5\times10^{-4}$ attains the highest return ($99.1 \pm 17.2$), improving over the single-frame policy ($m=1$, $77.9\pm 7.9$). Longer histories do not help further: with $m=16$, the return is lower and varies strongly across seeds, and with the larger learning rate of $10^{-3}$, both $m=8$ and $m=16$ plateau at lower returns. Notably, the single-frame MLP in our pipeline already exceeds the score of $\approx 30$ reported by~\citet{zakka2025mujoco}, so we attribute most of this gap to our training setup rather than to frame stacking. The same configuration performs best on the Berkeley humanoid, where it attains $24.1 \pm 0.6$, compared to $\approx 17$ reported by~\citet{zakka2025mujoco} for a single-frame MLP. We therefore use $m=8$ with learning rate $5\times10^{-4}$ as the baseline for both tasks in Fig.~\ref{fig:results_humanoids}.

\subsection{Co-state coefficient ablation on the Berkeley humanoid}\label{app:costate_coef}

\begin{figure}[h]
    \centering
    \includegraphics[width=0.8\linewidth]{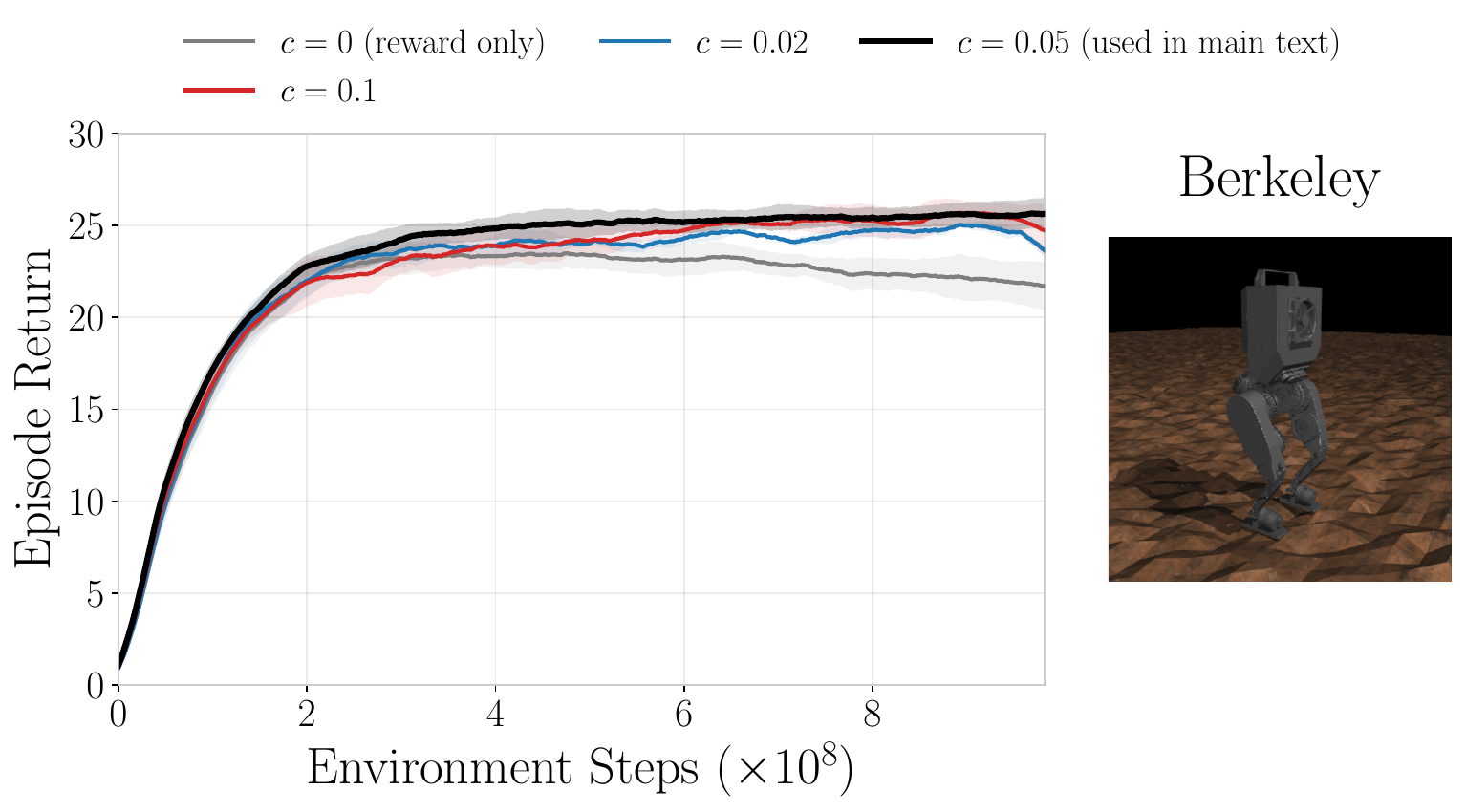}
    \caption{\textbf{Co-state coefficient sweep for the GRU on the Berkeley humanoid.} Mean $\pm$ std over 10 seeds for $c_{\text{co-state}} \in \{0, 0.05\}$ (the runs in Fig.~\ref{fig:results_humanoids}), 5 seeds for $c_{\text{co-state}} = 0.1$, and 3 seeds for $c_{\text{co-state}} = 0.02$.}
    \label{fig:costate_coef_berkeley}
\end{figure}

To select the co-state coefficient for the humanoid experiments, we train the GRU policy with a linear readout on the Berkeley humanoid for $c_{\text{co-state}} \in \{0, 0.02, 0.05, 0.1\}$, keeping all other hyperparameters as in Table~\ref{tab:hyperparameters}. Fig.~\ref{fig:costate_coef_berkeley} shows that every non-zero coefficient improves the return over reward-only training. The coefficient $c_{\text{co-state}} = 0.05$ attains the highest final return; we apply it to the H1 without further tuning. 

\section{Implementation and training details} \label{app:experimentdetails}
Here we provide an overview of the algorithmic components, architectural choices, and hyperparameters used in our implementation to ensure reproducibility. Our systems and default hyperparameter configurations follow standard continuous control setups, as implemented by MuJoCo Playground~\citep{zakka2025mujoco}. We leverage parts of the RSL-RL framework for the implementation of recurrent PPO~\citep{schwarke2025rsl}. The actor and critic are trained with their own parameters, and the actor outputs the mean of a Gaussian distribution. The log standard deviation ($\log \sigma$) is maintained as a state-independent learnable array initialized to zero. 

\subsection{FitzHugh-Nagumo system}
We evaluate our approach on a controlled FitzHugh-Nagumo oscillator. Setting the system parameters to $a=0, b=0, c=1$, and the control gain to $R_{\text{sys}}=1$, the continuous-time dynamics are given by:
\begin{equation}
\begin{aligned}
    \dot{v} &= v - \frac{v^3}{3} - w + u, \\
    \dot{w} &= v,
\end{aligned}
\end{equation}
where $z = [v, w]^\top$ is the system state and $u \in [-1, 1]$ is a bounded continuous control input. The objective is to minimize a standard quadratic cost over a finite horizon $T=10$:
\begin{equation}
    J = 0.1 \|z_T\|^2 + \int_0^{T} \left( 0.1 \|z_t\|^2 + 0.05 \|u_t\|^2 \right) \dd t.
\end{equation}

The system is integrated using a 4th-order Runge-Kutta (RK4) method with a physics substep of $5 \times 10^{-3}$ and a control period of $dt = 0.02$, resulting in episodes of 500 time steps. Initial states are sampled uniformly from $[-3, 3]^2$. To enforce partial observability, the agent only receives the scaled voltage variable ($v_t / 2$) as its observation, requiring it to infer the gating variable $w_t$ from history. 

Policies are trained using 1,024 parallel environments for a total of 153.6 million time steps. The recurrent policies use a hidden state size of 64 and are optimized using a learning rate of $5 \times 10^{-4}$. Gradients are computed using full backpropagation through time (BPTT) over the entire 500-step episode without truncation.

\subsection{Algorithmic Components} \label{sec:app_algorithmic}
Our optimization relies on PPO, integrating several standard mechanisms and code-level optimizations for stabilized training described in~\citet{huang202237}:
\begin{itemize}
    \item \textbf{Observation and Reward Standardization:} We maintain running statistics (mean and variance) for observations. Raw observations are normalized dynamically and strictly clipped to $[-10, 10]$ before being passed to the networks. Additionally, raw rewards (negative costs) are scaled dynamically by dividing them by the standard deviation of the running discounted returns: $r^{\text{scaled}}_t = r^{\text{raw}}_t / \sqrt{\text{Var}(R) + \epsilon}$, where $\epsilon = 10^{-8}$, and are subsequently clipped to $[-10, 10]$.
    \item \textbf{Advantage Estimation and Minibatch Normalization:} Generalized Advantage Estimation (GAE) is computed backwards through time over the unrolled trajectories. The resulting advantages are normalized (subtracting the mean and dividing by the standard deviation) at the minibatch level during the PPO update epoch, rather than globally over the entire rollout buffer.
    \item \textbf{Sequential Minibatches:} Our implementation constructs minibatches by fetching sequential trajectory segments to maintain the temporal dependencies required for BPTT.
       
    \item \textbf{Action Clipping:} During environmental interaction, actions sampled from the actor's distribution $a_t \sim \mathcal{N}(\mu_t, \sigma^2)$ are clipped to the valid environmental bounds $[u_{\min}, u_{\max}]$ before being applied to the physics step. The unclipped actions and corresponding log probabilities are stored for the PPO update.
    \item \textbf{Clipped Surrogate Objective \& Value Loss:} The policy is updated using the standard clipped PPO objective with a clipping parameter $\epsilon_{\text{clip}}$. The value loss also employs clipping to constrain the value function update.
    \item \textbf{Co-state loss and observation encoder}: Earlier versions of this work considered the use of a cosine similarity loss~\citep{leeftink2026neural}, instead of the linear probing loss adopted here. The cosine similarity objective requires that the co-state and its target are in the same space, which is resolved by adding an encoder to the observation that projects it into the same dimensionality as the hidden states. The encoder is non-linear and part of the forward architecture, however, while the linear probe loss applies a linear transformation that only affects the co-state loss.
\end{itemize}

\subsection{Hyperparameters}
The hyperparameter configurations for our systems are detailed in Table~\ref{tab:hyperparameters}. 

\begin{table}[h]
    \centering
    \caption{PPO and Network Hyperparameters}
    \label{tab:hyperparameters}
    \vspace{0.1in}
    \begin{tabular}{@{}llc@{}}
        \toprule
        \textbf{Category} & \textbf{Hyperparameter} & \textbf{Value} \\
        \midrule
        \multirow{3}{*}{\textbf{Architecture}} 
        & Hidden Layer Size & 256 \\
        & Initialization Scale & 1.0 (General), 0.01 (Actor Out) \\
        \midrule
        \multirow{4}{*}{\textbf{Optimization}} 
        & Optimizer & Adam  \\
        & Learning Rate & $1 \times 10^{-4}$ (CheetahRun), $3 \times 10^{-4}$ (H1, Berkeley), \\
        &  & $5 \times 10^{-4}$ (FHN) \\
        & Max Gradient Norm &  1. \\
        \midrule
        \multirow{8}{*}{\textbf{PPO Details}} 
        & Rollout Batch Size & 1024 (CheetahRun, FHN), 4096 (H1, Berkeley) \\
        & Minibatches & 32 (CheetahRun), 8 (H1, Berkeley, FHN) \\
        & PPO Epochs & 16 (CheetahRun), 5 (H1, Berkeley, FHN) \\
        & Discount Factor ($\gamma$) & 1. (FHN), 0.99 (H1, Berkeley, CheetahRun) \\
        & GAE Parameter ($\lambda$) & 0.95 \\
        & Clipping Epsilon ($\epsilon_{\text{clip}}$) & 0.2 \\
        & Value Loss Coef ($c_{\text{vf}}$) &  1. \\
        & Entropy Coef ($c_{\text{ent}}$) & 0.01 \\
        & Co-state Coef ($c_{\text{co-state}}$) &  0.2 (FHN), 0.1 (CheetahRun), 0.05 (humanoids) \\
        \bottomrule
    \end{tabular}
\end{table}

\subsection{Hardware}
Our experiments were conducted using Quadro RTX 6000 and A100 GPUs. For the FitzHugh-Nagumo and CheetahRun, training a single seed takes 4--5 hours depending on the recurrent cell on a Quadro RTX 6000. For the H1 and Berkeley simulations, training a single seed takes 10--12 hours on an A100.

\end{document}